\documentclass{article}
\usepackage{fullpage}

\usepackage{amsmath, amsthm, amssymb, natbib, graphicx, url}
\usepackage[utf8]{inputenc} 
\usepackage{hyperref}       
\usepackage{url}            
\usepackage{booktabs}       
\usepackage{amsfonts}       
\usepackage{nicefrac}       
\usepackage{microtype}      
\usepackage{xcolor}         
\usepackage{color-edits}
\usepackage{tablefootnote}

\usepackage[ruled, linesnumbered]{algorithm2e}
\usepackage{multirow}
\usepackage{array}
\usepackage{enumitem}
\usepackage{authblk}
\addauthor{na}{blue}
\addauthor{ks}{red}
\addauthor{bb}{magenta}
\newtheorem{theorem}{Theorem}
\newtheorem{definition}{Definition}
\newtheorem{example}{Example}
\newtheorem{lemma}{Lemma}
\newtheorem{corollary}{Corollary}

\numberwithin{equation}{section}

\def\ic{\lambda}

\DeclareMathOperator*{\argmin}{argmin}
\DeclareMathOperator*{\argmax}{argmax}

\def\reals{{\mathbb R}}
\newcommand{\eps}{\varepsilon}
\newcommand{\defeq}{\triangleq}

\newcommand{\innerprod}[2]{\langle #1, #2 \rangle}

\newcommand{\newreptheorem}[2]{%
\newenvironment{rep#1}[1]{%
 \def\rep@title{#2 \ref{##1}}%
 \begin{rep@theorem}}%
 {\end{rep@theorem}}}
\newreptheorem{theorem}{Theorem}
\newreptheorem{lemma}{Lemma}
\newreptheorem{proposition}{Proposition}
\newreptheorem{claim}{Claim}
\newreptheorem{corollary}{Corollary}
\newreptheorem{mainlemma}{Main Lemma}

\usepackage{prettyref}

\newcommand{\savehyperref}[2]{\texorpdfstring{\hyperref[#1]{#2}}{#2}}
\newrefformat{eq}{\savehyperref{#1}{\textup{(\ref*{#1})}}}
\newrefformat{eqn}{\savehyperref{#1}{Equation~\ref*{#1}}}
\newrefformat{lem}{\savehyperref{#1}{Lemma~\ref*{#1}}}
\newrefformat{def}{\savehyperref{#1}{Definition~\ref*{#1}}}
\newrefformat{thm}{\savehyperref{#1}{Theorem~\ref*{#1}}}
\newrefformat{cor}{\savehyperref{#1}{Corollary~\ref*{#1}}}
\newrefformat{sec}{\savehyperref{#1}{Section~\ref*{#1}}}
\newrefformat{app}{\savehyperref{#1}{Appendix~\ref*{#1}}}
\newrefformat{ass}{\savehyperref{#1}{Assumption~\ref*{#1}}}
\newrefformat{ex}{\savehyperref{#1}{Example~\ref*{#1}}}
\newrefformat{fig}{\savehyperref{#1}{Figure~\ref*{#1}}}
\newrefformat{alg}{\savehyperref{#1}{Algorithm~\ref*{#1}}}
\newrefformat{rem}{\savehyperref{#1}{Remark~\ref*{#1}}}
\newrefformat{conj}{\savehyperref{#1}{Conjecture~\ref*{#1}}}
\newrefformat{prop}{\savehyperref{#1}{Proposition~\ref*{#1}}}
\newrefformat{proto}{\savehyperref{#1}{Protocol~\ref*{#1}}}
\newrefformat{prob}{\savehyperref{#1}{Problem~\ref*{#1}}}
\newrefformat{claim}{\savehyperref{#1}{Claim~\ref*{#1}}}

\def\A{{\mathcal A}}
\def\B{{\mathcal B}}

\def\K{{\mathcal K}}

\def\O{{\mathcal O}}

\def\S{{\mathcal S}}

\def\bbE{{\mathbb E}}

\title{Variance-Reduced Conservative Policy Iteration}

\author[1]{Naman Agarwal\thanks{namanagarwal@google.com}}
\author[2]{Brian Bullins\thanks{bbullins@purdue.edu}}
\author[3]{Karan Singh\thanks{karansingh@cmu.edu}}
\affil[1]{Google Research, Princeton}
\affil[3]{Tepper School of Business, Carnegie Mellon University}
\affil[2]{Department of Computer Science, Purdue University}
\date{}

\begin{document}

\maketitle

\begin{abstract}%
    We study the sample complexity of reducing reinforcement learning to a sequence of empirical risk minimization problems over the policy space. Such reductions-based algorithms exhibit local convergence in the function space, as opposed to the parameter space for policy gradient algorithms, and thus are unaffected by the possibly non-linear or discontinuous parameterization of the policy class. We propose a variance-reduced variant of Conservative Policy Iteration that improves the sample complexity of producing a $\varepsilon$-functional local optimum from $O(\varepsilon^{-4})$ to $O(\varepsilon^{-3})$. Under state-coverage and policy-completeness assumptions, the algorithm enjoys $\varepsilon$-global optimality after sampling $O(\varepsilon^{-2})$ times, improving upon the previously established $O(\varepsilon^{-3})$ sample requirement.
\end{abstract}



\section{Introduction}
Reinforcement learning agents interact with the environment by adaptively executing actions with the goal of maximizing a cumulative long-term reward. A persistent challenge for such agents is operating in situations that involve large or continuous state spaces. Such large-scale Markov Decision Processes (MDPs) are accompanied by both the statistical challenge of generalization across states and the computational challenge of working with a large decision set, since often the policy class used in conjunction is too large to enumerate efficiently.

One approach to deal with these issues is to reduce reinforcement learning (RL) to a sequence of better-understood and easier learning problems like supervised learning (SL) or empirical risk minimization (ERM). Indeed, such an approach is well-studied (e.g. \cite{kakade2002approximately,kakade2003sample,langford2003reducing}; see \cite{ajks} for a modern treatment), and has since inspired widely-used practical variants \citep{schulman2015trust, schulman2017proximal, vieillard2020deep}. A typical algorithm of this kind incrementally updates the policy in small steps, using the solutions obtained from a SL/ERM oracle, to arrive at a policy that is a local minimum in the space of policies as opposed to a local minimum in the space of policy parameters. This \textit{functional} local optimality is a stronger notion of local optimality which holds independently of how the policy class itself may be parameterized, and is hence distinct from local convergence in parameter space. See below for a more detailed comparison of these notions. Existing structural results (e.g. Theorem 14.4 in \cite{ajks}) then translate such functional local optimality to global optimality guarantees under various appropriate state-coverage assumptions.

In this paper we study and improve the sample complexity of such reductions-based algorithms. Specifically, consider the Conservative Policy Iteration (CPI) \citep{kakade2002approximately} algorithm which reduces RL to weighted multi-class classification problems over the policy space. For any policy class $\Pi$, the CPI algorithm guarantees upon sampling $O\left(\frac{\log |\Pi|}{\varepsilon^4}\right)$\footnote{We handle infinite policy classes via the notion of covering numbers defined further in the paper.} transitions outputs a policy $\bar{\pi}$ which satisfies $\varepsilon$-local optimality in the function space. The notion of functional local optimality over a policy space $\Pi$ can be defined by considering the gradient of the value function $\nabla V^{\pi}$ with respect to the policy $\pi$ \footnote{As defined later in the paper we consider policies to be in the space $\reals^{S \times A}$ and thus the gradient lives in that space too}. A policy $\bar{\pi}$ is an $\epsilon$-functional local optimum if $ \max_{\pi'\in \Pi}\langle \nabla V^{\bar{\pi}}, \pi'- \bar{\pi} \rangle \leq \varepsilon.$

Our main result (formalized as Theorem \ref{thm:mainlocal}) is an algorithm that retains the above-stated functional local optimality guarantee while drawing $\tilde{O}\left(\frac{\log |\Pi|}{\varepsilon^3}\right)$ transitions. This reduction in sample complexity happens due to careful reuse of samples across successive calls to an ERM oracle, much like Experience Replay \citep{schaul2015prioritized}. A crucial feature here is that this reuse of samples does not employ trajectory-wise (or marginal) importance weights whose size may be uncontrollable, but instead uses a momentum-like weighting scheme inspired by variance-reduction methods \citep{cutkosky2019momentum} for non-convex optimization.

\newcolumntype{L}{>{\centering\arraybackslash}m{3cm}}
\newcolumntype{R}{>{\centering\arraybackslash}m{3cm}}
\newcolumntype{D}{>{\arraybackslash}m{5cm}}
\begin{table*}
    \begin{tabular}{@{}p{\textwidth}@{}}
        \centering
        \bgroup
        \def\arraystretch{2}
        \begin{tabular}{L|R|R|D}
            & Without variance reduction & With variance reduction & Remarks \\
            \hline
            Local optimality in parameter space (Policy Gradient) & \vspace{-0.8cm}$$\frac{1}{\varepsilon^4}$$ \citep{ghadimi2013stochastic} & \vspace{-0.8cm}$$\frac{1}{\varepsilon^3}$$ \citep{xu2019sample,shen2019hessian} & Assuming a smoothly parameterized policy class and possibly a variance bound on trajectory-wise importance weights. \\
            \hline
            Local optimality in function space & \vspace{-0.8cm}$$\frac{\log |\Pi|}{\varepsilon^4}$$ (CPI; \cite{kakade2002approximately}) & \vspace{-0.8cm}$$\frac{\log |\Pi|}{\varepsilon^3}$$ (Theorem~\ref{thm:mainlocal}) & Holds regardless of the parameterization of the policy class $\Pi$; without any state-coverage completeness assumption. \\
            \hline
            Global optimality given a restart distribution $\nu$ & \vspace{-0.8cm}$$\frac{\log |\Pi|}{\varepsilon^4}$$ (CPI; \cite{kakade2002approximately}) & \vspace{-0.8cm}$$\frac{\log |\Pi|}{\varepsilon^3}$$ (Corollary~\ref{cor:mainglobal}) & Assuming $D_\infty=\left\|\frac{d^{\pi^*}}{\nu}\right\|_\infty<\infty$; upto a policy-completeness residual term (Definition~\ref{def:disma}).\\
            \hline
            Global optimality given an exploratory policy class & \vspace{-0.8cm}$$\frac{\log |\Pi|}{\varepsilon^3}$$ (CPI; \cite{kakade2002approximately}, \cite{brukhim2022a}) & \vspace{-0.8cm}$$\frac{\log |\Pi|}{\varepsilon^2}$$ (Theorem~\ref{thm:fast}) &  $C_\infty=\max_{\mathrm{CH(\pi)}\in \Pi}\left\|\frac{d^{\pi^*}}{d^\pi}\right\|_\infty<\infty$; upto a policy-completeness residual term (Definition~\ref{def:disma2}).\\
            \hline
        \end{tabular}
        \egroup
    \end{tabular}
    \caption{The sample complexity of various optimality objectives, suppressing  $\texttt{poly}(|A|,1/(1-\gamma))$. Furthermore in the paper we dont assume that the policy class is finite but rather prove our results with $|\Pi_{\varepsilon}|$ which is the $\varepsilon$-covering number of $\Pi$. For exploratory policy classes, the stated sample complexity results from an improved iteration complexity of the CPI algorithm \citep{kakade2002approximately} as observed in \cite{brukhim2022a}, by directly appealing to the gradient domination characteristics of the value function.}
    \label{tab:table1}
\end{table*}

\textbf{Contrasting local optimality in parameter and function spaces.} Typically, policy gradient methods guarantee convergence to local optima in the parameter space. We contrast these notions below.

\begin{enumerate}[topsep=0pt,itemsep=-1ex,partopsep=1ex,parsep=1ex]
    \item Local optimality in function space is a stronger notion. Consider the policy class $\Pi=\{\pi(\theta): \theta\in \Theta\}$ where $\pi(\theta): \Theta \times \reals^{S \times A}$ is a differentiable function and a possibly non-convex loss function $V:\mathcal{F}\to \reals$. Via the application of a chain rule we see that 
    $\nabla_{\theta} V(\pi(\theta)) = \nabla V^{\pi} \frac{\partial \pi(\theta)}{\partial \theta}$. It is now easy to verify that for any $\theta$
    $$ \max_{\pi'\in {\Pi}}\langle \nabla V^{\pi(\theta)}, \pi'-\pi\rangle \leq \varepsilon \implies \|\nabla_\theta V(\pi(\theta)) \| \leq O(\varepsilon) $$
    for smoothly parameterized policy classes. Yet as the following example  demonstrates the reverse implications does not hold true.
    \begin{example}
\label{ex:minimas}
Consider a singleton dataset $(x,y) = (1,1)$, loss function $l(y, \hat{y}) = -(\hat{y}-y)^2/2$, and function class $\mathcal{F}=\{f_\theta(x) = \theta^2 x: \theta\in [-10,10]\}$. $\theta=0$ is a local optima in the parameter space, i.e. $\nabla_\theta l(y, f_\theta(x)) = -2(\theta^2x-1)\theta x = 0$, yet $f_\theta =0$ is not a local optima in the function space. In particular, $\nabla_f l(y,f_0(x))=1$ and $\nabla l(y, f_0(x)) \times (f_1(x) - f_0(x))=1\neq 0$. 
\end{example}
    \item For any composite optimization problem $\min_\theta V(\pi(\theta))$, local optimality in function space decouples the non-convexity induced by the policy parameterization $\pi(\theta)$ from the intrinsic non-convexity of the evaluation function $V$. Consequently, function-space algorithms \citep{mason1999boosting} are unaffected when the policy class is non-linearly (like neural nets) or discontinuously (like decision trees) parameterized. Meanwhile, parameter-space algorithms like SGD have to contend with the local optima introduced by non-linear parameterizations of the function class even for convex loss functions. 
    
    \item In the context of RL, unlike parameter-space local convergence, functional local optimality guarantees readily translate to global guarantees under appropriate state-coverage assumptions. See below. 
\end{enumerate}

\paragraph{Functional Local Optimality to Global Optimality}
Using standard local-to-global results in RL literature (e.g. Lemma~\ref{lem:transfernu} \& Lemma~\ref{lem:transferpi} in this work, or Theorem 14.4 in \cite{ajks}), we translate our functional local optimality result to a $\tilde{O}\left(\frac{\log |\Pi|}{\varepsilon^3}\right)$ sample complexity for global optimality (upto a policy-completeness term) when the learner has access to an exploratory distribution $\mu$ with sufficient state coverage. This is formalized in Theorem \ref{cor:mainglobal} and the coverage assumptions are formalized in Definition \ref{def:disma}.

Furthermore, when the policy class under consideration has sufficient state coverage by itself, we show that it is possible to skip this local-to-global approach, and directly guarantee an improved sample complexity of $\tilde{O}\left(\frac{\log |\Pi|}{\varepsilon^2}\right)$ for global convergence, which is tight even for the the subcase of binary classification. This is formalized as Theorem \ref{thm:fast} and the policy class coverage is formalized as Definition \ref{def:disma2}.

\section{Related work}
Reduction from reinforcement learning to supervised learning has been studied extensively, beginning with \cite{kakade2002approximately, kakade2003sample, langford2003reducing, NIPS2003_3837a451}. The local functional descent view of Conservative Policy Iteration (CPI) \citep{kakade2002approximately} is presented in detail in \cite{kakade2003sample, scherrer2014local}. In particular, CPI runs for $1/\varepsilon^{2}$ rounds and requires an $\varepsilon$-accurate supervised learning oracle in each round, thereby needing $O(1/\varepsilon^{4})$ samples in total.

In recent years, the technique of variance reduction \citep{schmidt2017minimizing, defazio2014saga, johnson2013accelerating} has led to faster rates for convex optimization in the finite-sum setting. The idea behind these methods is to use a (mini-)batch of gradients at a point to help reduce the variance of the stochastic estimator at subsequent points. These techniques have since been combined with ideas such as acceleration \citep{shalev2014accelerated, lin2015universal, allen2017katyusha}, and they have even led to improvements for non-convex problems \citep{allen2016variance, reddi2016stochastic} in terms of reaching $\eps$-approximate stationary points, i.e., points with gradient norm bounded by $\eps$. For the stochastic non-convex setting, variance reduced approaches have, in a series of works \citep{xu2018first, allen2018neon2, nguyen2017sarah, fang2018spider, cutkosky2019momentum}, led to improvements over the classical rate of $O(1/\eps^4)$ \citep{ghadimi2013stochastic}, ultimately reaching a rate of $O(1/\eps^3)$, which is tight under a mean-squared smoothness property of the noisy gradient \citep{arjevani2019lower}. Furthermore, such applications have proven useful in the context of constrained non-convex optimization, whereby similar improvements have been shown in \citep{reddi2016stochasticFW, zhang2020one}.

Following the analysis of SGD for nonconvex objectives \citep{ghadimi2013stochastic}, it follows that REINFORCE \citep{sutton1999policy} converges to a local minimum in parameter space after $\O(1/\varepsilon^{4})$ samples. A recent thread \citep{shen2019hessian, xu2019sample,zhang2021on,papini2018stochastic} of research applies variance-reduction techniques to improve upon this to $O(1/\varepsilon^3)$.
\section{Problem setting}
We setup a few notations of common use, and then proceed to delineate the interaction model. 

\subsection{Basic definitions}
A differentiable function $f: \K \to \reals$ is defined to be $L$-smooth, over a domain $\K$ with respect to a norm $\|\cdot\|$ and an inner product $\langle \cdot,\cdot\rangle$, if for every $x,y\in \K$ the inequality stated below holds. We note that the norm $\|\cdot\|$ need not necessarily be the canonical norm associated with the inner product $\langle \cdot, \cdot\rangle$. 
$$ \lvert f(y) - f(x) - \langle \nabla f(x), y-x \rangle \rvert \leq \frac{L}{2}\|x-y\|^2$$

For any $N\in \mathbb{N}$, let $\Delta_N$ be the unit simplex in $N$ dimensions, i.e. $\Delta_N = \{x:\in \reals_{+}^N : \|x\|_1 = 1\}$. For matrices in $\reals^{M\times N}$, define the dual-norm\footnote{See Lemma~\ref{lem:dualnorm} for a proof of the duality} pair $\|\cdot\|_{1,\infty}$ and $\|\cdot\|_{\infty,1}$ as follows. Note that the latter norm $\|\cdot\|_{1,\infty}$ is not a vector-induced matrix norm.
$$ \|A\|_{\infty,1}=\max_{1\leq i\leq M} \sum_{j=1}^N |A_{i,j}|, \qquad \|A\|_{1,\infty}=\sum_{i=1}^M \max_{1\leq j\leq N} |A_{i,j}|.$$

Let $\B_{\infty, 1}, \B_{1, \infty} \subset \reals^{S \times A}$ ($S,A$ correspond to the size of the state space and action space defined in the next section) denote the unit norm balls for the two norms over $S \times A$ matrices. For any element $\pi \in \B_{\infty, 1}$, and any $s \in S$, we will use the notation $\pi(s) \in \reals^{|A|}$ to denote the row corresponding to the state $s$. In subsequent sections we consider the decision set to be $\K=\Delta_{A}^{S}\subset \B_{\infty, 1}$, which is the cartesian product of $S$ unit simplices $\Delta_{A}$, to represent the space of stochastic policies over state space $\S$ and action set $\A$. Note that for any $\pi\in \Delta_{A}^{S}$, $\|\pi\|_{\infty,1}=\max_{s\in \S} \sum_{a\in \A} \pi(a|s)=1$.

For any $\varepsilon\geq 0$ and policy class $\Pi$, the $\varepsilon$-covering number $|\Pi_\varepsilon|$ is the minimum size of a set $\Pi_\varepsilon$ such that for every $\pi\in \Pi$, there exists a $\pi'\in \Pi_\varepsilon$ such that $\|\pi-\pi'\|_{\infty,1}\leq \varepsilon$. Such covering numbers are typically scale logarithmically as function of $\epsilon$ and henceforth we assume the same. We define $\mathrm{CH(\Pi)}$ to be the convex hull of $\Pi$. 



\subsection{Interaction model}
A Markov Decision Process (MDP) is a decision making framework specified by state space $\S$ with $|\S| \defeq S$, action space $\A$ with $|\A| \leq A$, a reward function $r:\S\times \A \to [0, 1]$, a transition kernel $P: \S\times \A \to \Delta_{A}$, discount factor $\gamma\in [0, 1)$, and an initial state distribution $\rho\in \Delta_{S}$. Define an {\em effective} horizon of $H_\gamma=\frac{1}{1-\gamma}$. A stochastic policy $\pi\in \Delta_{A}^{S}$ prescribes a choice of actions at each state $s$ as $a\sim \pi(\cdot | s)$. The execution of such a policy on the MDP induces a distribution over the space of trajectories, where each trajectory $\tau=(s_0,a_0,\dots)$ is a random variable sampled as
\[ s_0 \sim \rho, a_i \sim \pi(\cdot | s_i), s_{i+1} = P( \cdot | s_i, a_i)\]
Thus, averaging with respect to this distribution, it is possible to ascribe a value to every state-action pair as 
$$ Q^\pi(s,a) \defeq \mathbb{E}_{\tau} \left[ \sum_{t=0}^\infty \gamma^t r(s_t, a_t) \bigg| \pi,s_0=s, a_0=a\right] \text{ and } V^\pi(s) \defeq \mathbb{E}_{a\sim \pi(\cdot|s)} \left[Q^\pi(s,a) | \pi,s\right].$$
Further we will also consider a setting wherein the start state of the MDP, $s_0$ could be sampled from an arbitrary distribution $\mu$. Under this setting we define the following quantities,  
$$V^\pi_\mu \defeq \mathbb{E}_{s\sim \mu} \left[V^\pi(s)| \pi\right] \text{ and } d^\pi_\mu(s,a) = (1-\gamma) \sum_{t=0}^\infty \gamma^t \mathbb{E}_{s_0\sim\mu}\mathbb{E} \left[\mathbf{1}_{s_t=s \wedge a_t=a}|\pi,s_0\right].$$
$V^\pi_\mu$ and $d^{\pi}_{\mu}$ above captures the expected return and the (discounted) state distribution associated with a policy $\pi$ when the initial state is sampled from a distribution $\mu$. To ease the presentation, through the paper for quantities that depend on start state distributions, if the start state distribution is not explicitly specified, it is assumed to be the canonical start state disctirbution $\rho$. Thus, $V^\pi = V^\pi_\rho$ and $d^\pi = d^\pi_\rho$. Given a single argument $s$, we shall use the unary function $d^\pi_{\mu}(s)$ to denote the state-space marginal of $d^\pi_{\mu}(s,a)$, i.e. $d^\pi_{\mu}(s) = \sum_{a \in \A} d^\pi_{\mu}(s,a)$. Given a single state $s \in \S$, $d^{\pi}_{s}$ represents the (discounted) state distribution achieved when starting from a fixed state $s$.  

In subsequent sections, we consider two distinct ways in which a learner might interact with the MDP. In the \textbf{$\mu$-reset model}, the learner has access to an exploratory restart distribution $\mu\sim \Delta_{S}$, and can draw a trajectory of finite length from the MDP starting with an initial state $s_0\sim \mu$ sampled from $\mu$. Nevertheless, the objective for the learner still remains to maximize $V^\pi = V^\pi_\rho$.

The other alternative considered is the \textbf{episodic model} where the learner can draw trajectories from the MDP with respect to the canonical start state distribution $\rho$. 

\subsection{Computational model}
Since this work considers a reduction-based approach to RL, we assume that the learner is aided by the following computational oracles. Note that the tolerance parameter in the following definition scales naturally with the maximum possible loss.

\begin{definition}[ERM Oracle]
\label{def:ERMOracle}
Let $\mathcal{L}=\{x\to l^\top x: l\in \reals^{A}\}$ be the class of linear loss function over the decision set $\Delta_{A}$. Given a dataset $D=\{(s_i,l_i)\}_{i=1}^N$ where each example is in $\S\times \mathcal{L}$, and a tolerance $\varepsilon_{\mathrm{ERM}}>0$, the Empirical Risk Minimization (ERM) oracle $\texttt{ERM}$ outputs a policy $\pi\in \Pi$ such that
  $$ \sum_{i=1}^N l_i^\top \pi(s_i) \geq \max_{\pi^*\in\Pi}\sum_{i=1}^N l_i^\top \pi^*(s_i)- \varepsilon_{\mathrm{ERM}} \cdot \sum_{i=1}^N\|l_i\|_{\infty}.$$
\end{definition}
Previous approaches based on reducing RL to better understood subroutines sometimes also make use of a supervised learning oracle where in the dataset is replaced by a distribution over examples. Note that it is always possible to construct an ERM oracle (for any tolerance $\varepsilon_{\mathrm{ERM}}$) using a supervised learning oracle, by resampling with replacement. 

Through the main paper we use $\tilde{O}()$ to contain polynomial factors in problem constants in particular including $(1-\gamma)^{-1}, A$. In the appendix we restate our main results including all such dependency.   


\section{Main results}
Our main result is a variance-reduced algorithm (Algorithm~\ref{alg:cpi}) which we formally describe and explain in the next section. In this section we present the guarantees we prove. We begin by stating our main result in the $\mu$-reset model for any start state distribution $\mu$. Note that the episodic model is a natural sub-case whence the start state distribution is set to be $\rho$, the canonical start state distribution.

\begin{theorem}[Local optimality in function space]
  \label{thm:mainlocal}
  There exists an algorithm such that given any start state distribution $\mu$ input to the algorithm and any given $\varepsilon, \delta$, the algorithm produces a policy $\bar{\pi}$ which satisfies the following with probability $1-\delta$,
  $$ \max_{\pi'\in \Pi} \langle \nabla V^{\bar{\pi}}_\mu, \pi' - \bar{\pi} \rangle \leq \varepsilon $$
  Furthermore the algorithm samples at most $ \tilde{O}(\log(|\Pi_{\varepsilon}|)/\varepsilon^{3})$ episodes of expected length $O\left(\frac{1}{1-\gamma}\right)$ from the MDP starting from the start state distribution $\mu$.
\end{theorem}

The above theorem improves the sample complexity of such functional local convergence from the best known $\tilde{O}(\varepsilon^{-4})$ for the CPI algorithm to $\tilde{O}(\epsilon^{-3})$. As highlighted in Table \ref{tab:table1}, such improvements only exist for local optimality in parameter space which can be a significantly weaker guarantee. 

\paragraph{Global Optimality Results} Next, we demonstrate how such functional local optimality may be translated to a statement on global convergence with certain state-coverage and policy-completeness assumptions which we define next. These local-to-global translation lemmas (Lemma~\ref{lem:transfernu}; proven in the appendix for completeness) were first noted in \cite{kakade2002approximately}, along with \cite{scherrer2014local} who formally introduced the policy completeness notion. Our contribution (Corollary~\ref{cor:mainglobal}) here is an improvement in the dependence of the sample complexity on $\varepsilon$ as a consequence of an improved local functional optimality result (Theorem~\ref{thm:mainlocal}).
\begin{definition}
    \label{def:disma}
  Let $\pi^*$ be an optimal policy for the MDP in consideration. Given $\mu$, a start-state distribution the learner can draw from, define the distribution mismatch coefficient $D_\infty^{\mu}$ as stated below. Further, given a policy class $\Pi$, define $\epsilon_{\Pi,\mu}$ as a quantitative measure of policy completeness.
  $$ D_\infty^{\mu} \defeq \left\| \frac{d^{\pi^*}}{\mu} \right\|_\infty, \quad
  \epsilon_{\Pi,\mu} =  \max_{\pi\in \mathrm{CH(\Pi)}}\min_{\pi^*\in \Pi} \mathbb{E}_{s\sim d^\pi_\mu} \left[ \max_{a\in A} Q^\pi(s,a) - Q^\pi(s,\cdot)^\top \pi^*(s) \right].$$
\end{definition}
The distribution mismatch coefficient measures how exploratory the restart distribution $\mu$ is, and the associated policy completeness notion is a policy analogue of inherent bellman error \citep{munos2008finite}. The latter measures the degree to which a policy in $\Pi$ can best approximate the bellman optimality operator in an average sense with respect to the state distribution. Under the above definitions, we show the following lemma, which relates the global optimality gap of a policy to the local optimality measure of the policy.
\begin{lemma}
  \label{lem:transfernu}
  For any state distribution $\mu$ and any policy $\pi\in \mathrm{CH(\Pi)}$, the following holds 
  $$ V^* - V^\pi \leq \frac{D_\infty^{\mu}}{1-\gamma} \left( \max_{\pi'\in \Pi} \langle \nabla V^{{\pi}}_\mu, \pi' - {\pi}\rangle + \frac{1}{1-\gamma} \epsilon_{\Pi,\mu} \right).$$
\end{lemma}
In particular the above lemma shows that if one has access to a start state distribution with state coverage (i.e. $D^{\mu}_{\infty}$ is small) local optimality implies global optimality upto the policy completeness measure. We now provide sample complexity bounds for achieving global optimality via the following corollary of Theorem~\ref{thm:mainlocal} which translates the functional local optimality guarantee stated in the latter to the associated global optimality guarantee under appropriate policy coverage assumptions when the learner has access to an exploratory restart distribution $\mu$, with the aid of Lemma~\ref{lem:transfernu}. 

\begin{corollary}[Global optimality]
  \label{cor:mainglobal}
  There exists an algorithm such that given any start state distribution $\mu$ input to the algorithm and any given $\varepsilon, \delta$, the algorithm produces a policy $\bar{\pi}$ which satisfies the following with probability $1-\delta$, 
  $$ V^* - V^{\bar{\pi}} \leq \varepsilon + \frac{D_\infty^{\mu} \epsilon_{\Pi,\nu}}{(1-\gamma)^2}.$$
 Furthermore the algorithm samples at most $\tilde{O}\left(\frac{(D_\infty^{\mu})^3\log |\Pi_{\varepsilon}|}{\varepsilon^3}\right)$ episodes of expected length $O\left(\frac{1}{1-\gamma}\right)$ from the MDP starting from the start state distribution $\mu$.
\end{corollary}

\paragraph{Improved rates with an exploratory policy class}
Next, we provide improved rates of convergence to global optimality under the assumption that the underlying policy class by itself has sufficient overlap with the state distribution of an optimal policy. {\em Under this assumption we no longer require access to a state-distribution with sufficient coverage, so we state the results in the more restricted episodic setting, where every sample episode begins at a state sampled from the MDP's canonical start state distribution $\rho$.} For brevity we define $\epsilon_{\Pi} = \epsilon_{\Pi,\rho}$ to be the policy completeness with respect to the canonical start state distribution $\rho$. Formally, the policy class' overlap condition is captured by the following definition. 
\begin{definition}
    \label{def:disma2}
    Consider a policy class $\Pi$.
    Let $\pi^*$ be an optimal policy for the MDP in consideration. Define the distribution mismatch coefficient as 
    $ C_\infty = \max_{\pi\in \mathrm{CH(\Pi)}} \left\| \frac{d^{\pi^*}}{d^\pi} \right\|_\infty. $
\end{definition}
In such cases, the value function is approximately gradient dominated as the following lemma shows and it is possible to forego the procedure of arriving at a local optimality guarantee before making claims on the global suboptimality.  
\begin{lemma}
    \label{lem:transferpi}
    For any $\pi\in \mathrm{CH(\Pi)}$, the following holds 
    $$ V^* - V^\pi \leq C_\infty\left( \max_{\pi'\in \Pi} \langle \nabla V^{{\pi}}, \pi' - {\pi}\rangle + \frac{1}{1-\gamma} \epsilon_{\Pi} \right).$$
\end{lemma}

A more direct analysis yields the following result which improves the sample complexity to be scaling as $\tilde{O}(\varepsilon^{-2})$.

\begin{theorem}[Faster global optimality]\label{thm:fast}
There exists an algorithm such that given any $\varepsilon, \delta$, the algorithm produces a policy $\bar{\pi}$ which satisfies the following with probability $1-\delta$, 
\[ V^{*} - V^{\bar{\pi}} \leq \varepsilon + \frac{C_{\infty} \epsilon_{\Pi}}{1 - \gamma}.\]
Furthermore the algorithm samples at most $ \tilde{O}(C_\infty^2\log(|\Pi_{\varepsilon}|)/\varepsilon^{2})$ episodes of expected length $O\left(\frac{1}{1-\gamma}\right)$ sampled from the MDP starting at the canonical start state distribution $\rho$. 
\end{theorem}

\section{Algorithms}
\subsection{Variance-reduced Conservative Policy Iteration}
In this section we present our main algorithm \textbf{Variance-reduced Conservative Policy Iteration}. The algorithm is formally described as Algorithm \ref{alg:cpi}. The algorithm takes as input a start state distribution $\mu$ for the MDP and assumes access to an ERM oracle (Definition \ref{def:ERMOracle}) over the policy class $\Pi$. The algorithm is parameterized by parameters $\eta, \ic, \gamma, T, \varepsilon_{\mathrm{ERM}}$, wherein notably $\varepsilon_{\mathrm{ERM}}$ is the accuracy target for the ERM oracle calls and the algorithm samples a total of $3T$ episodes from the MDP starting from the start state distribution $\mu$. The algorithm makes use of two important subroutines the Q-sampler (Algorithm~\ref{alg:qsamp}) and H-sampler (Algorithm~\ref{alg:hsamp}) to compute unbiased estimates of linear forms in the functional policy gradient and quadratic forms in the functional policy Hessian respectively. These sub-routines are formally stated in the appendix (Algorithms \ref{alg:qsamp} and \ref{alg:hsamp}) and we state their properties in this section. We now proceed to explain the construction of Algorithm \ref{alg:cpi}. 

\paragraph{Conservative Policy Iteration} The core structure of our proposed algorithm follows the idea proposed by the Conservative Policy Iteration(CPI) algorithm \citep{kakade2002approximately} which maintains a policy $\pi_t$. At every step a new candidate policy $\pi'_{t}$ is obtained via the following 
\begin{equation}
    \label{eqn:cpi_update}
    \pi'_{t} = \argmin \innerprod{\nabla V^{\pi_t}}{\pi}.
\end{equation}
Note that $\nabla V^{\pi_t}$ is the functional policy gradient. The following lemma from \cite{sutton1999policy} provides a concrete estimator for the functional policy gradient. 
\begin{lemma}
  \label{lem:grad}[\cite{sutton1999policy}]
  For any policy pair $\pi,\pi'$ and start-state distribution $\mu$, we have
  $$ \langle \nabla V^\pi_\mu, \pi'\rangle = \frac{1}{1-\gamma} \mathbb{E}_{s\sim d^\pi_\mu} \left[Q^\pi(s,\cdot)^\top \pi'(s) \right].$$
\end{lemma}

Using the above, it can be observed that the minimization problem \eqref{eqn:cpi_update} can be cast as a supervised learning problem over the policy space. CPI obtains the next policy via a convex combination $\pi_{t+1} = \eta \pi_t + (1 - \eta)\pi'_{t+1}$. This step can be interpreted as a step of the Frank-Wolfe algorithm \citep{frank1956algorithm} in the policy space. Via standard convergence analyses (see eg. \cite{hazan2016introduction}) one can expect to require a number of iterations scaling with $\varepsilon^{-2}$. To solve the per-step optimization problem, the sample complexity scales with $\varepsilon^{-2}$, leading to a total sample complexity scaling with $\varepsilon^{-4}$. In order to improve sample complexity, we instead employ a variance-reduced estimator of the gradient $\nabla V^{\pi_t}$. 

\paragraph{Variance Reduction} The variance-reduced estimator of the functional policy gradient we use is an adaptation of the momentum based low-variance estimator STORM proposed by \cite{cutkosky2019momentum}. For a stochastic function $F(x) = \bbE[f(x, \xi)|x]$, \cite{cutkosky2019momentum} propose the following estimator
\begin{equation}
\label{eqn:storm_estimator}
    v_t \defeq (1-\ic)v_{t-1} + \ic \nabla f(x_t, \xi_t) + (1-\ic) (\nabla f(x_t, \xi_t)  - \nabla f(x_{t-1}, \xi_t))
\end{equation}
The proposed estimator runs a running average of the previous stochastic gradient estimates (akin to momentum) with a key addition of a term accounting for gradient differences viz. $\nabla f(x_t, \xi_t) - \nabla f(x_{t-1}, \xi_t)$.
It is critical that the random variable $\xi_t$ is the same on both gradient estimates in the above expression. It is therefore tricly to estimate the gradient difference term for RL applications as the distribution of the stochastic gradient depends on the current policy (trajectories are sampled from the current policy), an issue highlighted in the previous works \citep{xu2019sample,shen2019hessian}. We avoid the issue by noting that this term can also be estimated via the following expression involving the Hessian of $f$, 
\begin{equation}
    \label{eqn:storm-hessian}
    \nabla f(x_t, \xi_t) - \nabla f(x_{t-1}, \xi_t) \sim \nabla^2 f((1-b)x_t + bx_{t-1}, \xi_t)(x_{t}-x_{t-1}),
\end{equation}
where $b$ is uniformly sampled from $[0,1]$. We employ this Hessian-based approach for our problem. We note that this \textit{correction} term (either computed via gradient difference or a Hessian-vector product) is necessary in a sense for reducing the complexity of stochastic optimization in general \citep{arjevani2019lower,arjevani2020second}. The Hessian based estimator requires building estimates for the functional Hessian-vector products (in the policy space) for which we provide a sub-routine H-Sampler (Algorithm \ref{alg:hsamp}). We now proceed to describe the construction of our algorithm. 

\begin{algorithm}[t]
      \textbf{Input}: Initial state distribution $\mu$, ERM oracle $\texttt{ERM}$, parameters $\eta, \ic, \gamma, T, \varepsilon_{\mathrm{ERM}}$. \\
      Choose an initial policy $\pi_0=\pi_1\in \Pi$ arbitrarily. \\
      Create an empty dataset $D_0$ supported over state-linear-loss pairs $\{(s_i,\widehat{l}_i):s_i\in S, \widehat{l}_i\in \reals^{A}\}$. \\
      \For{$t=1$ {\bfseries to} $T$}{
        $(s_t, \widehat{Q}_t)\in S\times \reals^{A} \leftarrow $ Q-sampler (Algorithm~\ref{alg:qsamp}) with inputs $\pi_t, \mu, \varepsilon_Q$.\\
        Sample $b\sim \texttt{Unif}([0,1])$, and define a policy $\bar{\pi}_t = (1-b)\pi_t + b\pi_{t-1}$.\\
        $(s'_t, s''_t, \widehat{H}_t)\in S\times S\times \reals^{A\times A} \leftarrow $ H-sampler (Algorithm~\ref{alg:hsamp}) with inputs $\bar{\pi}_t, \mu, \varepsilon_H$.\\
        Create a new dataset $D_t$ by first multiplying each existing loss vector in $D_{t-1}$ by $(1-\ic)$:
        $$ D_t = \{(s_i, (1-\ic)l_i): (s_i,l_i)\in D_{t-1}\}, $$
        and then appending the following three tuples to the said new dataset $D_t$.
        \begin{align*}
            \left(s_t, \widehat{l}_{t,1}=\frac{\ic}{1-\gamma}\widehat{Q}_t\right), \quad \left(s'_t, \widehat{l}_{t,2}=\frac{\gamma(1-\ic)}{(1-\gamma)^2}\widehat{H}_t \left(\pi_t(s''_t)-\pi_{t-1}(s''_t)\right)\right), \\
            \left(s''_t, \widehat{l}_{t,3}=\frac{\gamma(1-\ic)}{(1-\gamma)^2}\widehat{H}_t \left( \pi_t(s'_t)- \pi_{t-1}(s'_t) \right)\right).
        \end{align*}\\
      Call the ERM oracle with tolerance $\varepsilon_{\mathrm{ERM}}$ on the dataset $D_t$ to obtain $\pi'_t=\texttt{ERM}(D_t).$ \\
      Update $\pi_{t+1} = (1-\eta)\pi_{t} + \eta \pi'_t$. \\
      Empirically compute local improvement of $\pi'_t$ in comparison to $\pi_t$ with respect to $D_t$ as \\
      $$ \widehat{A}_t = \sum_{(s,\widehat{l})\in D_t} \widehat{l}^\top \left(\pi'_t(s) - \pi_t(s)\right).$$
      }
  \textbf{Option 1}: \textbf{return}  $\bar{\pi} = \pi_{t'}$ where $t'=\argmin_{T/2\leq t\leq T} \widehat{A}_t$. \\
  \textbf{Option 2}: \textbf{return}  $\bar{\pi} = \pi_{T}$.
  \caption{Variance-reduced Conservative Policy Iteration}\label{alg:cpi}
  

\end{algorithm}

\begin{algorithm}[h!]

\textbf{Input}: Policy $\pi$, and start-state distribution $\mu$, tolerance $\varepsilon_Q$. \\
      Sample the initial state $s_0\sim \mu$. \\
      \For{$t=0,\dots \infty$}{
        With probability $1-\gamma$, record the current state as $s_t$ and exit the loop. \\
        Else, draw an action $a_t\sim \pi(s_t)$ and transition to the next state $s_{t+1}\sim P(\cdot | s_t, a_t)$.
      }
      Draw the $t^{th}$ action $a_t\sim \texttt{Unif}(A)$ uniformly, and observe the next state $s_{t+1}\sim P(\cdot|s_t,a_t).$ \\
     \For{$t'=t+1,\dots, \infty$}{
        With probability $1-\gamma$, record $R=\frac{r(s_{t'},a_{t'})}{1-\gamma}$,   and construct the vector in $\reals^{A}$ as 
        $$ \widehat{Q}(a) = \mathbb{I}_{a=a_t}AR$$ 
        and return the tuple $(s_t, \widehat{Q})$ ending the algorithm's execution. \\
        Else, draw an action $a_{t'}\sim \pi(s_{t'})$ to transition to the next state $s_{t'+1}\sim P(\cdot | s_{t'}, a_{t'})$.
      }
  \caption{Q-sampler}\label{alg:qsamp}

\end{algorithm}
\paragraph{Construction of the algorithm}

Our overall algorithm is based on building an estimator of $\nabla V^{\pi_t}$ akin to the $v_t$ estimator from \ref{eqn:storm_estimator}. With such an estimator we wish to solve the minimization problem
\begin{equation}
\label{eqn:storm-fw}
    \pi'_{t} = \argmin_{\pi \in \Pi} \innerprod{v^{t}}{\pi}.
\end{equation}
We approach the above problem via reduction to ERM problem over the policy class and solve it via mapping the estimators to datasets of $(s, l)$ pairs over which we solve the ERM problem. Concretely at every step of Algorithm \ref{alg:cpi} we call the Q-sampler which we show to have the following guarantee 

\begin{lemma}
  \label{lem:qsamp}
  Q-sampler (Algorithm~\ref{alg:qsamp}) when run with a policy $\pi$ and start-state distribution $\mu$ produces a random tuple $(s, \widehat{Q}) \in  \S\times \reals^{A}$ with the distribution $s\sim d^\pi_\mu$, and $ \mathbb{E}\left[\widehat{Q} \Big| s,\pi\right] = Q^\pi (s,\cdot)$. This in particular implies that for any $\pi'$,
  $$ \langle \nabla V^\pi_\mu , \pi'\rangle = \frac{1}{1-\gamma} \mathbb{E}_{(s,\widehat{Q})\sim\mathcal{A}}\left[\widehat{Q}^\top \pi'(s)\right].$$
  Furthermore, $\left\|\widehat{Q}\right\|_1 \leq \frac{|A|}{1-\gamma}$, and the expected length of sampled episode is at most $\frac{3}{1-\gamma}$.
\end{lemma}

Therefore we include $s_t$ and a scaled version of $\widehat{Q}_t$ in our dataset. To see why the scaling is $\frac{\lambda}{(1 - \gamma)}$, note that the $\frac{1}{(1 - \gamma)}$ from the scaling in the gradient expression above and the scaling of $\ic$ comes from the usage in the estimator \eqref{eqn:storm_estimator}. Similarly to estimate the Hessian term akin to \eqref{eqn:storm-hessian} in the estimator, we mix $\pi_{t-1}, \pi_t$ to obtain $\bar{\pi}_t$ and we call the H-Sampler subroutine. For the H-Sampler subroutine we show the following guarantee 

\begin{lemma}
    \label{lem:hsamp}
    H-sampler (Algorithm~\ref{alg:hsamp}) when run with a policy $\pi$ and start-state distribution $\mu$ produces a random tuple $(s,s',\widehat{H})\in \S \times \S\times \reals^{A\times A}$ such that for any policy pair $\pi',\pi''$,
   $$   \langle \pi'', \nabla V^\pi_\mu \pi'\rangle = \frac{\gamma}{(1-\gamma)^2} \mathbb{E}_{(s,s',\widehat{H})\sim\mathcal{A}}\left[\pi''(s)^\top \widehat{H} \pi'(s') + \pi'(s)^\top \widehat{H} \pi''(s') \right].$$
    Furthermore, $\sum_{a,a'\in \A\ \times \A}|\widehat{H}(a,a')| \leq \frac{A^2}{1-\gamma}$ and the expected episode length is at most $\frac{5}{1-\gamma}$.
\end{lemma}

Further we account for the scalings arising both from the Hessian guarantee and the usage in the estimator when including the $\widehat{H}_t$ into our dataset. As a result we see that at all times $t$ our dataset construction satisfies the following lemma 
\begin{lemma}
     \label{lem:varmain}
    For all $t \geq 0$, let $v_t\in \reals^{S\times A}$ be a vector defined recursively such that for any  $\pi\in \B_{\infty, 1}$, such that $\langle v_0, \pi\rangle=0$, and 
    \begin{align*} \langle v_t, \pi\rangle \defeq & (1-a) \langle v_{t-1}, \pi\rangle + \frac{\ic}{1-\gamma}\widehat{Q}_t^\top \pi (s_t) \\&+ \frac{\gamma(1-\ic)}{(1-\gamma)^2} \left( (\widehat{H}_t(\pi_t(s''_t)-\pi_{t-1}(s''_t)))^\top \pi(s'_t) +(\widehat{H}_t(\pi_t(s'_t)-\pi_{t-1}(s'_t)))^\top \pi(s''_t)  \right).
    \end{align*}
       We have that for all times $t\leq T$, and for any policy $\pi$, the ERM loss for $\pi$ on $D_t$ can be expressed as 
        $$ \langle v_t,\pi\rangle = \sum_{(s, l) \in D_t} \widehat{l}^\top \pi(s). $$ 
    \end{lemma}

The above lemma ensures that all times $t$, $\pi'_t$ is the solution of the minimization problem described in \eqref{eqn:storm-fw} for an estimator $v_t$ which by using Lemma \ref{lem:qsamp} and Lemma \ref{lem:hsamp} can readily be seen to be the same as the STORM estimator \ref{eqn:storm_estimator} but in the functional space.

\subsection{H-Sampler}

\begin{algorithm}[h!]

    \textbf{Input}: Policy $\pi$, and start-state distribution $\mu$, tolerance $\varepsilon_H$. \\
    Sample the initial state $s_0\sim \mu$. \\
    \For{$t=0,\dots \infty$}{
      With probability $1-\gamma$, record the current state as $s_t$ and exit the loop. \\
      Else, draw an action $a_t\sim \pi(s_t)$ and transition to the next state $s_{t+1}\sim P(\cdot | s_t, a_t)$.
    }
    Draw the $t^{th}$ action $a_t\sim \texttt{Unif}(A)$ uniformly, and observe the next state $s_{t+1}\sim P(\cdot|s_t,a_t).$ \\
    \For{$t'=t+1,\dots \infty$}{
      With probability $1-\gamma$, record the current state as $s'_{t'}$ and exit the loop. \\
      Else, draw an action $a_{t'}\sim \pi(s_{t'})$ and transition to the next state $s_{t'+1}\sim P(\cdot | s_{t'}, a_{t'})$.
    }
    Draw the $t'^{th}$ action $a_{t'}\sim \texttt{Unif}(A)$ uniformly, and observe the next state $s_{t'+1}\sim P(\cdot|s_{t'},a_{t'}).$ \\
    \For{$t''=t'+1,\dots, \infty$}{
      With probability $1-\gamma$, record $R=\frac{r(s_{t''}, a_{t''})}{1-\gamma}$ construct the matrix in $\reals^{|A|\times |A|}$ as 
      $$ \widehat{H}(a_1,a_2) = \mathbb{I}_{a_1=a_{t}\wedge a_2=a_{t'}}|A|^2R$$ 
      and return the tuple $(s_t, s_{t'}, \widehat{H})$ ending the algorithm's execution. \\
      Else, draw an action $a_{t''}\sim \pi(s_{t''})$ to transition to the next state $s_{t''+1}\sim P(\cdot | s_{t''}, a_{t''})$.
    }
\caption{H-sampler}\label{alg:hsamp}

\end{algorithm}

The Q-Sampler and H-Sampler, which are importance sampling based estimators, to are stated as algorithms \ref{alg:qsamp} and \ref{alg:hsamp} respectively. We highlight the salient aspects of the proposed H-Sampler which we believe to be of independent interest. To obtain the H-Sampler, similar to the case of gradient in Lemma \ref{lem:grad}, we provide an explicit characterization of the functional policy Hessian. To define the Hessian, we make use of the notion of a future advantage $F^\pi(s,a|\pi')$ of a policy $\pi'$ with respect to a baseline policy $\pi$, when starting from some state $s$ and action $a$. Intuitively, it represents the value of playing one step of a candidate policy $\pi'$ at a random step (geometrically distributed) in the future when starting from a state-action pair $(s,a)$, all the while executing a baseline policy $\pi$.

\begin{definition}
    For any policy pair $\pi, \pi'$, define the future advantage $F^\pi(\cdot|\pi'):\S\times \A \to \reals$ of a policy $\pi'$ with respect to a baseline policy $\pi$ as  
    $$ F^\pi(s,a|\pi') =  \mathbb{E}_{s'\sim P(\cdot | s,a)} \mathbb{E}_{s''\sim d^\pi_{s'}} \left[ Q^\pi(s'',\cdot)^\top \pi'(s'') \right]. $$
\end{definition}
Note that the future advantage $F^\pi(\cdot|\pi')$ is linear in $\pi'$. Lemma \ref{lem:hess} provides a characterization of the functional Hessian of the value function as a bi-linear operator over $\pi',\pi''$. The interchangability of the roles of $\pi'$ and $\pi''$ ensures the symmetry of the bi-linear form.

\begin{lemma}
    \label{lem:hess}
    For any policy triplet $\pi,\pi',\pi''$ and start-state distribution $\mu$, we have
    $$ \langle \pi'', \nabla^2 V^\pi_\mu \pi'\rangle = \frac{\gamma}{(1-\gamma)^2} \left(\mathbb{E}_{s\sim d^\pi_\mu} \left[F^\pi(s,\cdot|\pi')^\top \pi''(s) + F^\pi(s,\cdot|\pi'')^\top \pi'(s) \right]\right).$$
\end{lemma}

The following lemma shows the main guarantee for our H-Sampler, which immediately implies Lemma \ref{lem:hsamp} using Lemma \ref{lem:hess}. 

\begin{lemma}
    \label{lem:hsamp_sub}
    H-sampler (Algorithm~\ref{alg:hsamp}) when run with a policy $\pi$ and start-state distribution $\mu$ produces a random tuple $(s,s',\widehat{H})\in \S \times \S\times \reals^{A\times A}$ such that $s\sim d^\pi_\mu$,  and for any $\pi'$ \[\mathbb{E}\left[\widehat{H}^\top \pi'(s') \Big| s,\pi\right] = F^\pi (s,\cdot|\pi').\]
\end{lemma}


\section{Overview of Analysis}
Due to space constraints, we defer the analysis and proofs of the theorem entirely to the appendix, where the theorems are restated with the correct parameter instantiations. We provide a high level summary of the analysis approach here. Overall, the core of our algorithmic approach and analysis resembles the one-sample stochastic Frank-Wolfe algorithm proposed by \cite{zhang2020one} which also employs the STORM estimator for variance reduction in stochastic optimization. However the RL setting and especially performing the variance reduction in functional space brings some unique challenges which we tackle in our analysis. In particular, the functional(policy) space is bounded in $\infty$-norm with gradients bounded in $1$-norm. Thereby, the variance reduction properties of the STORM algorithm which are naturally stated in $\ell_2$ norms need to be extended to $\infty, 1$ norms. To this end we provide an alternative analysis of STORM which bounds the deviation between the estimator and the true gradient with \textit{high probability}(as opposed to smaller variance) over a covering set of the policy space. We believe this alternative analysis extending STORM to $\infty, 1$ norms and establishing high probability guarantees can be of independent interest. Furthermore to construct the estimator in the functional settings for RL, we devise novel functional Hessian-vector product oracle, which requires developing a sampling based expression for the Hessian-vector product (summarized in Lemma \ref{lem:hess}).

\section{Conclusion}
We revisit the problem of reducing reinforcement learning to a sequence of ERM problems. Using ideas from variance reduction in stochastic optimization, we improve the sample complexity of achieving a functional local optimum in policy space from $O(\varepsilon^{-4})$ to $O(\varepsilon^{-3})$. As we discuss, functional local optimum guarantees can be significantly stronger than parameter space local optimum guarantees, which we demonstrate by translating our improved sample complexity results for functional local optimum to improved bounds for global optimality under state coverage assumptions.

\bibliographystyle{plainnat}
\bibliography{references.bib}

\appendix

\section{Detailed statements of theorems}
\subsection{Theorem \ref{thm:mainlocal}}
\label{sec:paramlocalthm}

\begin{theorem}[Theorem \ref{thm:mainlocal} detailed]
  \label{thm:mainlocalapp}
  For a given $\varepsilon, \delta$ define $\varepsilon_{\mathrm{cover}} \defeq \frac{\varepsilon(1-\gamma)^2}{80A}$ and define the function \[C(T) \defeq \frac{8 A^{3/2}\sqrt{\gamma \log(2T|\Pi_{\varepsilon_{\mathrm{cover}}}|/\delta)}}{(1-\gamma)^{5/2}}.\]
  Then Algorithm \ref{alg:cpi} when run with any parameters $\eta, T$ satisfying the following equations,
  \begin{equation}
  \label{eqn:mainconds}
      \eta \leq \frac{\varepsilon (1- \gamma)^3}{40 \gamma} \qquad  \eta T \geq \frac{(1-\gamma)}{2\gamma A}\log \left(\frac{1}{20\varepsilon (1-\gamma)^2} \right) \qquad \frac{2}{(1-\gamma) \eta T} + 5 C(T) \cdot \sqrt{\eta} \leq 3\varepsilon/10
  \end{equation}

    and $\varepsilon_{\mathrm{ERM}} = \frac{\varepsilon (1-\gamma)^2}{60A}$, then given any start state distribution $\mu$ input to the algorithm, the algorithm produces a policy $\bar{\pi}$ which satisfies the following with probability $1-\delta$,
  $$ \max_{\pi'\in \Pi} \langle \nabla V^{\bar{\pi}}_\mu, \pi' - \bar{\pi} \rangle \leq \varepsilon.$$
  
    Further there exists a setting of $T = \tilde{O}(\log(|\Pi_{\varepsilon}|) A^3(1 - \gamma)^{-6}\varepsilon^{-3})$ such that the conditions \eqref{eqn:mainconds} can be satisfied and therefore the algorithm samples at most $ \tilde{O}(\log(|\Pi_{\varepsilon}|) A^3(1 - \gamma)^{-6}\varepsilon^{-3})$ episodes of expected length $O\left(\frac{1}{1-\gamma}\right)$ from the MDP starting from the start state distribution $\mu$.
\end{theorem}
In the above theorem, $\tilde{O}$ hides polylogarithmic factors in the relevant parameters.

\subsection{Theorem \ref{thm:fast}}
\label{sec:paramfastthm}
\begin{theorem}[Theorem \ref{thm:fast} detailed]
  \label{thm:fastapp}
  For a given $\varepsilon, \delta$ define $\varepsilon_{\mathrm{cover}} \defeq \frac{\varepsilon(1-\gamma)^2}{80A}$ and define the function \[C(T) \defeq \frac{8 A^{3/2}\sqrt{\gamma \log(2T|\Pi_{\varepsilon_{\mathrm{cover}}}|/\delta)}}{(1-\gamma)^{5/2}}.\]
  Then Algorithm \ref{alg:cpi} when run with any parameters $\eta, T$ satisfying the following equations,
  \begin{equation}
  \label{eqn:mainconds2}
     \eta T \geq  2 C_{\infty} \log\left(\frac{10}{\varepsilon(1-\gamma)}\right) \qquad \eta \log\left( \frac{2T \Pi_{\varepsilon_{\mathrm{cover}}}}{\delta} \right) \leq  \frac{\varepsilon^2(1-\gamma)^5}{6400A^3} \qquad  \eta \leq \frac{\varepsilon(1-\gamma)^3}{40 \gamma C_{\infty}}
  \end{equation}

    and $\varepsilon_{\mathrm{ERM}} = \frac{\varepsilon (1-\gamma)^2}{20 A C_{\infty}}$, then starting from the canonical start state distribution $\rho$, the algorithm produces a policy $\bar{\pi}$ which satisfies the following with probability $1-\delta$,
  \[ V^{*} - V^{\bar{\pi}} \leq \varepsilon + \frac{C_{\infty} \varepsilon_{\Pi}}{1 - \gamma}.\]
  
    Further there exists a setting of $T = \tilde{O}(C_\infty^2 \log(|\Pi_{\varepsilon}|) A^3 (1-\gamma)^{-5} \varepsilon^{-2})$ such that the conditions \eqref{eqn:mainconds} can be satisfied and therefore the algorithm samples at most $ \tilde{O}(C_\infty^2 \log(|\Pi_{\varepsilon}|) A^3 (1-\gamma)^{-5} \varepsilon^{-2})$ episodes of expected length $O\left(\frac{1}{1-\gamma}\right)$ from the MDP starting from the start state distribution $\rho$.
\end{theorem}
In the above theorem, $\tilde{O}$ hides polylogarithmic factors in the relevant parameters. 
\section{Proofs of the Main Results}
\subsection{Proof of Theorem \ref{thm:mainlocalapp}}
\begin{proof}[Proof of Theorem~\ref{thm:mainlocalapp}]
    Let us first observe that value function $V^\pi_\mu$ is a smooth over the space of policies. The following statement holds independently of how the policy class itself may be parameterized.
    
    \begin{lemma}
      \label{lem:smooth}
      For any start-state distribution $\mu$, $V^\pi_\mu$ is $\frac{\gamma}{(1-\gamma)^3}$-smooth in the $\|\cdot\|_{\infty,1}$ norm, i.e. for any two policies $\pi, \pi'$, we have that
      \[| V_\mu^{\pi'} - V_\mu^\pi -  \langle \nabla V^\pi_\mu, \pi'-\pi \rangle | \leq \frac{\gamma}{(1-\gamma)^3}\|\pi' - \pi''\|_{\infty, 1}^2\]
      Further for any two policies $\pi, \pi'$ and any starting distribution $\mu$ we have that 
      \[ \innerprod{\nabla V_{\mu}^{\pi'} - \nabla V_{\mu}^{\pi''}}{\pi} \leq \frac{2\gamma}{(1-\gamma)^3} \|\pi' - \pi''\|_{\infty, 1}\]
    \end{lemma}

    We invoke smoothness of $V^\pi_\mu$, as Lemma~\ref{lem:smooth} certifies, to observe that since successive iterates are close in the $\|\cdot\|_{\infty,1}$ norm, we have 
    \begin{align}
      V^{\pi_{t+1}}_\mu &= V^{\pi_{t}+\eta(\pi'_t-\pi_t)}_\mu \nonumber\\
      &\geq  V^{\pi_t}_\mu + \eta \langle \nabla V^{\pi_t}_\mu, \pi'_t-\pi_t\rangle - \frac{\gamma \eta^2}{(1-\gamma)^3} \|\pi'_{t}-\pi_t  \|_{\infty,1}^2. \label{eqn:smoothV}
    \end{align}
    
    Next, we wish to use the fact to fact that $\pi'_t$ was chosen by a ERM oracle, and therefore approximately maximizes the inner product with the gradient of the value function. To do this, we first relate the ERM objective (as in Algorithm~\ref{alg:cpi}) to the said gradient. This result supplants Lemma~\ref{lem:varmain}.

    \begin{theorem}
     \label{thm:varmain}
    For all $t \geq 0$, define a sequence of vectors $v_t\in \reals^{S\times A}$ recursively as follows. Let $v_0$ be any vector such that for all $\pi\in \B_{\infty, 1}$, we have that $\langle v_0, \pi\rangle=0$. Further for any $\pi\in \B_{\infty, 1}, t > 0$, let $v_t$ be an vector satisfying,
    \begin{align*} \langle v_t, \pi\rangle \defeq & (1-\ic) \langle v_{t-1}, \pi\rangle + \frac{\ic}{1-\gamma}\widehat{Q}_t^\top \pi (s_t) \\&+ \frac{\gamma(1-\ic)}{(1-\gamma)^2} \left( (\widehat{H}_t(\pi_t(s''_t)-\pi_{t-1}(s''_t)))^\top \pi(s'_t) +(\widehat{H}_t(\pi_t(s'_t)-\pi_{t-1}(s'_t)))^\top \pi(s''_t)  \right).
    \end{align*}
    Here $s_t, s_t', s_t''$ are the sequence of states produced by the algorithm. We have that for all $t \geq 0$ the dataset $D_t$ maintained by Algorithm \ref{alg:cpi} satisfies the property that for any policy $\pi$, the ERM loss for $\pi$ on $D_t$ can be expressed as 
        $$ \sum_{(s, l) \in D_t} l^\top \pi(s) = \langle v_t,\pi\rangle  .$$ 
        Further for any $\eta < \frac{(1-\gamma)}{4A\gamma}$ setting $\ic = \frac{4\eta A\gamma}{(1 - \gamma)}$, we have that for any $\varepsilon, \delta$, with probability at least $1 - \delta$, for all policies $\pi \in \mathrm{CH}(\Pi)$ and time $t \leq T$ the following holds
        \[|\innerprod{v_t - \nabla V^{\pi_t}_\mu}{\pi}| \leq \frac{1}{(1-\gamma)^2}(1-\ic)^t + \frac{8 A^{3/2}\sqrt{\gamma \log(2T|\Pi_{\varepsilon}|/\delta)}}{(1-\gamma)^{5/2}} \cdot \sqrt{\eta} + \frac{4A \varepsilon}{(1-\gamma)^2}.\]
        
        Furthermore it holds with probability 1, that for all $t$, $\|v_t\|_{1, \infty} \leq \frac{2A}{(1-\gamma)^2}$.
    \end{theorem}
    
    For the rest of the proof define $\varepsilon_{\mathrm{cover}} \defeq \frac{\varepsilon(1-\gamma)^2}{80A}$ and define the function $C(T) \defeq \frac{8 A^{3/2}\sqrt{\gamma \log(2T|\Pi_{\varepsilon_{\mathrm{cover}}}|/\delta)}}{(1-\gamma)^{5/2}}$. Invoking Theorem \ref{thm:varmain} using $\varepsilon_{\mathrm{cover}}$ and using the conditions on $T$ in the statement of Theorem \ref{thm:mainlocalapp}, it can now be checked that for any $\delta$ and for any $t \in [T/2, T]$ we have that 
    \begin{equation}
        \label{eqn:ttt}
        |\innerprod{v_t - \nabla V^{\pi_t}_\mu}{\pi}| \leq \frac{\varepsilon}{10} + C(T) \cdot \sqrt{\eta}.
    \end{equation}
    
    Now the $\varepsilon_{\mathrm{ERM}}$-tolerant Empirical Risk Minimization Oracle, by its definition, guarantees for any $t$
    \begin{align*}
      \max_{\pi^*\in \Pi} \langle \nabla V^{\pi_t}_\mu, \pi^* - \pi'_t \rangle &\leq \max_{\pi^*\in \Pi} \langle v_t, \pi^* - \pi'_t \rangle + 2\max_{\pi\in \Pi} |\langle v_t -\nabla V^{\pi_t}_\mu , \pi \rangle | \\
      &= \max_{\pi^*\in \Pi}\sum_{i=0}^{3t} \widehat{l}_i^\top (\pi^*(s_i)-\pi'_t(s_i)) + 2\max_{\pi\in \Pi} |\langle v_t -\nabla V^{\pi_t}_\mu , \pi\rangle | \\
      &\leq \|v_t\|_{1, \infty}\varepsilon_{\mathrm{ERM}} + 2\max_{\pi\in \Pi} |\langle v_t -\nabla V^{\pi_t}_\mu , \pi\rangle | 
    \end{align*}
    
    Using \eqref{eqn:ttt}, with probability $1-\delta$, for all $t \in [T/2, T]$, the inequality concerning successive iterates may thus be written as 
    \begin{align*}
      \max_{\pi^*\in \Pi}\langle \nabla V^{\pi_t}_\mu, \pi^*-\pi_t\rangle =& \max_{\pi^*\in \Pi}\langle \nabla V^{\pi_t}_\mu, \pi^*-\pi'_t\rangle +\langle \nabla V^{\pi_t}_\mu, \pi'_t-\pi_t\rangle\\
    \leq& \|v_t\|_{1, \infty}\varepsilon_{\mathrm{ERM}} + 2\max_{\pi\in \Pi} |\langle v_t -\nabla V^{\pi_t}_\mu , \pi\rangle | + \langle \nabla V^{\pi_t}_\mu, \pi'_t-\pi_t\rangle \\
      \leq & \|v_t\|_{1, \infty}\varepsilon_{\mathrm{ERM}} + \frac{\varepsilon}{10} +  C(T) \cdot \sqrt{\eta} + \frac{V^{\pi_{t+1}}_\mu - V^{\pi_t}_\mu}{\eta} + \frac{4\gamma \eta}{(1-\gamma)^3}\\
      \leq & \|v_t\|_{1, \infty}\varepsilon_{\mathrm{ERM}} + \frac{\varepsilon}{5} +  C(T) \cdot \sqrt{\eta} + \frac{V^{\pi_{t+1}}_\mu - V^{\pi_t}_\mu}{\eta}.
    \end{align*}
    where the second last inequality uses \eqref{eqn:smoothV} and . that for any policy $\pi$, $\|\pi\|_{\infty,1}=1$, and therefore, $\|\pi'_t-\pi_t\|_{\infty,1}\leq 2$. The last inequality follows from the condition on $\eta$ in the theorem. Further since value functions are always bounded by $\frac{1}{1-\gamma}$, we average the inequality over iterations via telescoping to observe that with probability $1-\delta$, 
    \begin{align}
      \label{eq:avggua}
      \frac{2}{T}\sum_{t=T/2}^T\max_{\pi^*\in \Pi}\langle \nabla V^{\pi_t}_\mu, \pi^*_t-\pi_t\rangle &\leq   \frac{2A}{(1-\gamma)^2}\varepsilon_{\mathrm{ERM}} + \frac{\varepsilon}{5} +  C(T) \cdot \sqrt{\eta} + \frac{2}{(1-\gamma)\eta T}.
    \end{align}
    Now, finally, we move from an average to a guarantee on a specific iterate. From, we have that for any $t \in [T/2, T]$,
    \begin{align*}
      &\left\lvert  \widehat{A}_t - \max_{\pi^*\in\Pi}\langle \nabla V^{\pi_t}_\mu, \pi^*-\pi_t\rangle \right\rvert \\
      =& \left\lvert  \langle v_t, \pi'_t-\pi_t\rangle - \max_{\pi^*\in\Pi}\langle \nabla V^{\pi_t}_\mu, \pi^*-\pi_t\rangle \right\rvert\\
      =& \left\lvert \max_{\pi^*\in\Pi} \langle v_t, \pi^*-\pi_t\rangle - \max_{\pi^*\in\Pi}\langle \nabla V^{\pi_t}_\mu, \pi^*-\pi_t\rangle \right\rvert + \varepsilon_{\mathrm{ERM}}\\
      \leq& 2\max_{\pi\in \Pi} |\langle v_t -\nabla V^{\pi_t}_\mu , \pi \rangle | + \|v_t\|_{1, \infty} \varepsilon_{\mathrm{ERM}} 
    \end{align*}
    
    Let $\bar{t}=\argmin_{t \in [T/2,T]} \widehat{A}_t$ and $t'= \argmin_{t \in [T/2,T]} \max_{\pi^*} \langle \nabla V^{\pi_t}_\mu, \pi^* - \pi_t\rangle$. Then it follows by definition of $t'$ and $\bar{t}$ and the above inequality that
    \begin{align*}
      &\max_{\pi^*\in \Pi}\langle \nabla V^{\pi_{\bar{t}}}_\mu, \pi^*-\pi_{\bar{t}}\rangle \\
      \leq & \widehat{A}_{\bar{t}}+\frac{2A}{(1-\gamma)^2}\varepsilon_{\mathrm{ERM}} +  2\max_{t \in [T/2, T]}\max_{\pi\in \Pi} |\langle v_t -\nabla V^{\pi_t}_\mu , \pi \rangle |  \\
      \leq & \widehat{A}_{t'} +\frac{2A}{(1-\gamma)^2}\varepsilon_{\mathrm{ERM}} +  2\max_{t \in [T/2, T]}\max_{\pi\in \Pi} |\langle v_t -\nabla V^{\pi_t}_\mu , \pi \rangle | \\
      \leq & \min_{t \in [T/2, T]}\max_{\pi^*\in \Pi} \langle \nabla V^{\pi_t}_\mu, \pi^*-\pi_t\rangle+\frac{4A}{(1-\gamma)^2}\varepsilon_{\mathrm{ERM}} +  4\max_{t \in [T/2, T]}\max_{\pi\in \Pi} |\langle v_t -\nabla V^{\pi_t}_\mu , \pi \rangle | \\
      \leq & \frac{2}{T}\sum_{t=T/2}^T\max_{\pi^*\in \Pi} \langle \nabla V^{\pi_t}_\mu, \pi^*-\pi_t\rangle +\frac{4A}{(1-\gamma)^2}\varepsilon_{\mathrm{ERM}} +  4\max_{t \in [T/2, T]}\max_{\pi\in \Pi} |\langle v_t -\nabla V^{\pi_t}_\mu , \pi \rangle | \\
    \end{align*}
    where the second last inequality follows from the average iterate guarantee in Equation~\ref{eq:avggua}. Now combining the above, \eqref{eq:avggua} and \eqref{eqn:ttt} we get that for any $\delta$ the following holds with probability $1 - \delta$,
    \begin{align*}
      \max_{\pi^*\in \Pi}\langle \nabla V^{\pi_{\bar{t}}}_\mu, \pi^*-\pi_{\bar{t}}\rangle &\leq \frac{2}{(1-\gamma)\eta T} + \frac{6A}{(1-\gamma)^2}\varepsilon_{\mathrm{ERM}} + \frac{6\varepsilon}{10} + 5C(T) \cdot \sqrt{\eta} \\
      &\leq \varepsilon
    \end{align*}
\end{proof}

\section{Faster global convergence - Proof of Theorem~\ref{thm:fastapp}}
\begin{proof}[Proof of Theorem~\ref{thm:fast}]
We invoke smoothness of $V^\pi$, as Lemma~\ref{lem:smooth} certifies, to observe that since successive iterates are close in the $\|\cdot\|_{\infty,1}$ norm, we have 
        \begin{align*}
          V^{\pi_{t+1}} &= V^{\pi_{t}+\eta(\pi'_t-\pi_t)}\\
          &\geq  V^{\pi_t} + \eta \langle \nabla V^{\pi_t}, \pi'_t-\pi_t\rangle - \frac{\gamma \eta^2}{(1-\gamma)^3} \|\pi'_{t}-\pi_t  \|_{\infty,1}^2.
        \end{align*}
Using Theorem~\ref{thm:varmain} and the definition of  $\varepsilon_{\mathrm{ERM}}$-tolerant Empirical Risk Minimization Oracle, we have that for any $t$,
        \begin{align*}
            &\max_{\pi^*\in \Pi}\langle \nabla V^{\pi_t}, \pi^*-\pi_t\rangle \\
            =& \max_{\pi^*\in \Pi}\langle \nabla V^{\pi_t}, \pi^*-\pi'_t\rangle +\langle \nabla V^{\pi_t}, \pi'_t-\pi_t\rangle\\
          \leq& \frac{2A}{(1-\gamma)^2}\varepsilon_{\mathrm{ERM}} + 2\max_{\pi\in \Pi} |\langle v_t -\nabla V^{\pi_t} , \pi\rangle | + \langle \nabla V^{\pi_t}, \pi'_t-\pi_t\rangle 
          \end{align*}
Continuing on, using the above inequality and Lemma~\ref{lem:transferpi}, we have
\begin{align*}
   V^* -  V^{\pi_{t+1}} &\leq  V^*- V^{\pi_t} - \eta \langle \nabla V^{\pi_t}, \pi'_t-\pi_t\rangle + \frac{\gamma \eta^2}{(1-\gamma)^3} \|\pi'_{t}-\pi_t  \|_{\infty,1}^2 \\
   &\leq  V^*- V^{\pi_t} - \eta\max_{\pi^*\in \Pi} \langle \nabla V^{\pi_t}, \pi^*-\pi_t\rangle +\eta\left( \frac{2A}{(1-\gamma)^2}\varepsilon_{\mathrm{ERM}} + 2\max_{\pi\in \Pi} |\langle v_t -\nabla V^{\pi_t} , \pi\rangle | \right) +\frac{4\gamma \eta^2}{(1-\gamma)^3}  \\
   &\leq  V^*- V^{\pi_t} - \eta \frac{V^*-V^{\pi_t}}{C_\infty} + \eta \left(\frac{2A}{(1-\gamma)^2}\varepsilon_{\mathrm{ERM}} + 2\max_{\pi\in \Pi} |\langle v_t -\nabla V^{\pi_t} , \pi\rangle | +\frac{\varepsilon_\Pi}{1-\gamma}\right) +\frac{4\gamma \eta^2}{(1-\gamma)^3} \\
   =& \left(1-\frac{\eta}{C_\infty}\right)V^*- V^{\pi_t} + \eta \left(\frac{2A}{(1-\gamma)^2}\varepsilon_{\mathrm{ERM}} + 2\max_{\pi\in \Pi} |\langle v_t -\nabla V^{\pi_t} , \pi\rangle | +\frac{\varepsilon_\Pi}{1-\gamma}\right) +\frac{4\gamma \eta^2}{(1-\gamma)^3}.
  \end{align*}
  Unrolling the above inequality from $t=T/2$ to $t=T$ and noting that for any policy $\pi$, $V^* - V^{\pi} \leq \frac{1}{1-\gamma}$ we get that 
  \begin{align*}
      V^* -  V^{\pi_{T}} &\leq  \left(1-\frac{\eta}{C_\infty}\right)^{T/2} \frac{1}{1-\gamma} + C_\infty \left(\frac{2A}{(1-\gamma)^2}\varepsilon_{\mathrm{ERM}} + 2\max_{t \in [T/2, T]}\max_{\pi\in \Pi} |\langle v_t -\nabla V^{\pi_t} , \pi\rangle | +\frac{ \varepsilon_\Pi}{1-\gamma}\right) +\frac{4\gamma \eta C_\infty}{(1-\gamma)^3} \\
      &\leq \frac{3 \varepsilon}{10} + C_\infty \cdot 2\max_{t \in [T/2, T]}\max_{\pi\in \Pi} |\langle v_t -\nabla V^{\pi_t} , \pi\rangle | + \frac{ C_{\infty}\varepsilon_\Pi}{(1-\gamma)} \\
      &\leq \frac{3 \varepsilon}{10} + 2C_\infty C(T) \sqrt{\eta} + \frac{ C_{\infty}\varepsilon_\Pi}{(1-\gamma)}
  \end{align*}
  where the second last inequality follows from the conditions of the theorem and the last inequality follows from \eqref{eqn:ttt} which holds with probability $1-\delta$ for any $\delta$. Now using the conditions in the theorem give the requisite statement for any $\delta$ with probability at least $1-\delta$.
\end{proof}

\section{High Probability Bound for Gradient Estimator - Proof of Theorem \ref{thm:varmain}} 
\begin{proof}
The first part of the theorem follows immediately via the definition of $v_t$ and the definition of the datasets $D_t$ in Algorithm \ref{alg:cpi}. We now proceed with the bound on the deviation. Remember that $v_t$ is defined in a recursive fashion by satisfying the following for any $\pi \in \B_{\infty,1}$,

\begin{align*} \langle v_t, \pi\rangle \defeq & (1-\ic) \langle v_{t-1}, \pi\rangle + \ic \cdot  \frac{\widehat{Q}_t^\top \pi (s_t)}{1-\gamma} \\&+ (1-\ic) \left(\frac{\gamma}{(1-\gamma)^2} \left( (\widehat{H}_t(\pi_t(s''_t)-\pi_{t-1}(s''_t)))^\top \pi(s'_t) +(\widehat{H}_t(\pi_t(s'_t)-\pi_{t-1}(s'_t)))^\top \pi(s''_t)  \right) \right).
    \end{align*}
For brevity in the proof we define the following random functions defined over all $\pi \in \B_{\infty, 1}$ and for all $t \geq 0$
\[\psi_t(\pi) \defeq \frac{\widehat{Q}_t^\top \pi (s_t)}{1-\gamma}\]
\[\zeta_t(\pi) \defeq \frac{\gamma}{(1-\gamma)^2} \left( (\widehat{H}_t(\pi_t(s''_t)-\pi_{t-1}(s''_t)))^\top \pi(s'_t) +(\widehat{H}_t(\pi_t(s'_t)-\pi_{t-1}(s'_t)))^\top \pi(s''_t)  \right) \]
Therefore by definition we have that for all $\pi \in \B_{\infty, 1}$ and for all $t$,
\[\innerprod{v_t}{\pi} = (1-\ic)\innerprod{v_{t-1}}{\pi} + \ic \psi_t(\pi) + (1-\ic) \zeta_t(\pi).\]

Before moving onto the proof we will provide some simple upper bounds on the random variables $\psi_t(\pi), \zeta_t(\pi)$ for any $\pi \in \B_{\infty, 1}$. Using Claim \ref{lem:qsamp} we get that for any $t \leq T, \pi \in \B_{\infty, 1}$,
\begin{equation}
\label{eq:temppsi}
   |\psi_t(\pi)| \leq \frac{\|\hat{Q}_t\|_1 \|\pi(s_t)\|_{\infty}}{1-\gamma} \leq \frac{A}{(1-\gamma)^2}.
\end{equation}
Further using Claim \ref{lem:hsamp} we get that for any $t \leq T, \pi \in \B_{\infty, 1}$, 
\begin{align}
   |\zeta_t(\pi)| &\leq \frac{\gamma}{(1-\gamma)^2}\left(\sum_{a,a' \in \A \times \A} |H(a,a')|\right)\|\pi(s_t')\|_{\infty}\left(\|\pi_t(s_t') - \pi_{t-1}(s_t')\|_{\infty} + \|\pi_t(s_t'') - \pi_{t-1}(s_t'')\|_{\infty} \right) \nonumber\\
   &\leq \frac{2 A^2\gamma}{(1-\gamma)^3}\|\pi_t - \pi_{t-1}\|_{\infty} \leq \frac{2 A^2\gamma}{(1-\gamma)^3}\|\pi_t - \pi_{t-1}\|_{\infty, 1}.  \label{eq:tempzeta}
\end{align}
We now move on to the main proof. Let $\bbE_{t}$ represent expectation fixing all the randomness upto and including time $t$. Then we have using Lemma \ref{lem:qsamp} that $\pi \in \B_{\infty, 1}$
\[\bbE_{t-1}[\psi_t(\pi)] = \innerprod{\nabla V_{\mu}^{\pi_t}}{\pi}.\]
Further using Lemma~\ref{lem:hsamp} and the inputs to the H-sampler from Algorithm \ref{alg:cpi}, it follows that 
\[\bbE_{t-1}[\zeta_t(\pi)] = \bbE_{b_t}\left[\innerprod{\pi_{t} - \pi_{t-1}}{\nabla^2 V_{\mu}^{\bar{\pi}_t} \pi} \right] = \bbE_{b_t}\left[\innerprod{\nabla^2 V_{\mu}^{\bar{\pi}_t}(\pi_{t} - \pi_{t-1})}{ \pi} \right] = \innerprod{\nabla V_{\mu}^{\pi_t} - \nabla V_{\mu}^{\pi_{t-1}}}{\pi}.\]
Next consider the definitions of the following sequences for every $t \geq 0$ and $\pi \in \B_{\infty, 1}$, 
\[\epsilon_t \defeq v_t - \nabla V^{\pi_t}_{\mu} \qquad y_t(\pi) \defeq \frac{\innerprod{\epsilon_t}{\pi}}{(1-\ic)^t}.\]
We next show that for any $\pi \in \B_{\infty, 1}$, the sequence $\{y_t(\pi)\}$ is a martingale sequence over time $t$. To see this consider the following derivation, 
\begin{align*}
  \bbE_{t-1}[y_t(\pi)] &= \frac{\bbE_{t-1}[\innerprod{\epsilon_t}{\pi}]}{(1-\ic)^t} =\frac{\bbE_{t-1}[\innerprod{(v_t - \nabla V_{\mu}^{\pi_t})}{\pi}]}{(1-\ic)^t} \\
   &= \frac{1}{(1-\ic)^t} \left( (1-\ic) \bbE_{t-1}[\innerprod{(v_{t-1} - \nabla V_{\mu}^{\pi_{t-1}})}{\pi}] + \ic  \left(\bbE_{t-1}[\psi_t(\pi)] - \innerprod{\nabla V_{\mu}^{\pi_t}}{\pi}\right) \right.\\
   & \left. \qquad \qquad + (1-\ic) \left( \bbE_{t-1}[\zeta_t(\pi)] - \innerprod{(\nabla V_{\mu}^{\pi_t} - \nabla V_{\mu}^{\pi_{t-1}})}{\pi} \right) \right) \\
  &= \frac{(1-\ic)\innerprod{\epsilon_{t-1}}{\pi}}{(1-\ic)^t} = y_{t-1}(\pi).  
\end{align*}
We now wish to use Azuma's inequality to show concentration for the martingale sequence. 
Note that for any policies $\pi, \pi'$ Lemma \ref{lem:grad} implies that $\innerprod{\nabla V_{\mu}^{\pi}}{\pi'} \leq \frac{1}{(1-\gamma)^2}$. Using the above derivations and Lemma \ref{lem:smooth}, we can now bound the differences of the martingale sequences as follows which holds for any $\pi \in \B_{\infty, 1}$ and all $t$, 
\begin{align*}
    x_t(\pi) \defeq |y_t(\pi) - y_{t-1}(\pi)| &=  \frac{|\innerprod{\epsilon_t}{\pi} - (1-\ic)\innerprod{\epsilon_{t-1}}{\pi}|}{(1-\ic)^t}  \\
    &= \frac{|\ic(\psi_t - \innerprod{\nabla V_{\mu}^{\pi_t}}{\pi}) + (1-\ic)(\zeta_t - \innerprod{(\nabla V_{\mu}^{\pi_t} - \nabla V_{\mu}^{\pi_{t-1}})}{\pi}|}{(1-\ic)^t}\\
    &\leq  \frac{\ic \cdot\frac{2A}{(1-\gamma)^2} + (1-\ic)\cdot\frac{2A^2\gamma}{(1-\gamma)^3}\|\pi_t - \pi_{t-1}\|_{\infty, 1} }{(1-\ic)^t} \\ 
    &\leq \frac{2A}{(1-\gamma)^2}\cdot \frac{\ic + \eta\cdot\frac{4A\gamma }{(1-\gamma)} }{(1-\ic)^t}. 
\end{align*}
To use Azuma's inequality we need to control the sum of the worst-case differences. To this end consider the following which holds with probability 1 for any $\pi \in \B_{\infty, 1}$,
\begin{align*}
    \sum_{\tau=1}^{t} x_{\tau}^2(\pi) &\leq \sum_{\tau=1}^{t} \frac{4A^2}{(1-\gamma)^4}\cdot \frac{\left(\ic + \eta\cdot\frac{4A\gamma }{(1-\gamma)}\right)^2 }{(1-\ic)^{2\tau}} \\
    &\leq \sum_{\tau=1}^{t}  \frac{4A^2}{(1-\gamma)^4}\cdot \frac{\left(\ic + \eta\cdot\frac{4A\gamma }{(1-\gamma)}\right)^2}{(1-\ic)^{2\tau}} \\
    &\leq   \frac{4A^2}{(1-\gamma)^4}\cdot \frac{\left(\ic + \eta\cdot\frac{4A\gamma }{(1-\gamma)}\right)^2}{\ic(1-\ic)^{2t}}.
\end{align*}
A direct application of Azuma's inequality implies that for any $\pi \in \B_{\infty, 1}$ and any $t, \delta$ with probability at least $1 - \delta$ the following holds,
\[ |y_t(\pi)| \leq |y_0(\pi)|  + \frac{4A}{(1-\gamma)^2} \cdot \frac{\sqrt{(\sqrt{\ic} + \frac{\eta}{\sqrt{\ic}} \cdot \frac{4A\gamma }{(1-\gamma)} )^2\log(2/\delta)}}{(1-\ic)^{t}}  \]
Setting $\ic = \frac{4\eta A\gamma}{(1 - \gamma)}$, we get that for any $\pi \in \B_{\infty, 1}, t, \delta$, with probability at least $1 - \delta$ the following holds
    \[ |y_t(\pi)| \leq |y_0(\pi)|  + \frac{8 A^{3/2}\sqrt{\gamma\log(2/\delta)}}{(1-\gamma)^{5/2}} \cdot \frac{\sqrt{\eta}}{(1-\ic)^{t}}  \]
Replacing the definition of $y_t(\pi)$ we get that for any $\pi \in \B_{\infty, 1}, \delta, t$, with probability least $1 - \delta$, the following holds
\begin{align*}
    |\innerprod{\epsilon_t}{\pi}| &\leq |\innerprod{\epsilon_0}{\pi}|(1-\ic)^t + \frac{8 A^{3/2}\sqrt{\gamma\log(2/\delta)}}{(1-\gamma)^{5/2}} \cdot \sqrt{\eta} \\
    &\leq \frac{1}{(1-\gamma)^2}(1-\ic)^t + \frac{8 A^{3/2}\sqrt{\gamma\log(2/\delta)}}{(1-\gamma)^{5/2}} \cdot \sqrt{\eta}.
\end{align*}
Let $\varepsilon > 0$ be any number and $\Pi_{\varepsilon}$ be the associated covering set of the policy class $\Pi$. Using a union bound over all choices of $\pi \in \Pi_{\varepsilon}$ and all timesteps $t \leq T$ we get that for any $\varepsilon, \delta$, with probability at least $1 - \delta$, we have that for any policy $\pi \in \Pi_{\varepsilon}$ and any time $t \leq T$,
\begin{equation}
    \label{eqn:tempeq}
    |\innerprod{v_t - \nabla V_{\mu}^{\pi_t}}{\pi}| \leq \frac{1}{(1-\gamma)^2}(1-\ic)^t + \frac{8 A^{3/2}\sqrt{\gamma \log(2T|\Pi_{\epsilon}|/\delta)}}{(1-\gamma)^{5/2}} \cdot \sqrt{\eta}.
\end{equation}
For the rest of the argument we will generate a crude bound over $\|v_t\|_{1, \infty}$. To this end will bound $\innerprod{v_t}{\pi}$ for all $\pi \in \B_{\infty, 1}$. By Lemma \ref{lem:dualnorm} since $\|\cdot\|_{\infty,1}$ and $\|\cdot\|_{1, \infty}$ are duals of each other this will imply a bound on $\|v_t\|_{1, \infty}$. Therefore consider any $\pi \in \B_{\infty, 1}$. Using the definitions of $v_t, \psi_t$ and $\zeta_t$ we get that 
\begin{align*}
    \innerprod{v_t}{\pi} &\leq (1-\ic)\innerprod{v_{t-1}}{\pi} + \ic |\psi_t(\pi)| + (1-\ic)\zeta_t(\pi) \\
    &\leq (1-\ic)\innerprod{v_{t-1}}{\pi} + \ic \frac{A}{(1-\gamma)^2} + \eta\cdot\frac{4A^2\gamma}{(1-\gamma)^3} \\
 &= (1-\ic)\innerprod{v_{t-1}}{\pi} +  \eta\cdot\frac{8A^2\gamma}{(1-\gamma)^3},    
\end{align*}
where the last equality uses the choice of $\ic = \frac{4\eta A\gamma}{(1 - \gamma)}$. Now we will show by induction that for all $t$, $\innerprod{v_t}{\pi} \leq \frac{2A}{(1-\gamma)^2}$. The base case is immediate. For the inductive case, it follows via the following computation using the choice of $\ic$,
\begin{align*}
    \innerprod{v_t}{\pi} 
    &\leq (1-\ic)\innerprod{v_t}{\pi} + \eta\cdot\frac{8A^2\gamma}{(1-\gamma)^3} \\
    &\leq (1-\ic)\cdot \frac{2A}{(1-\gamma)^2}  +  \ic \cdot \frac{2A}{(1-\gamma)^2} \leq \frac{2A}{(1-\gamma)^2}.\\
\end{align*}
This implies that for all $t$, $\|v_t\|_{1,\infty} \leq \frac{2A}{(1-\gamma)^2}$. It can be shown using Lemma \ref{lem:grad} that for all $t$, $\|V_{\mu}^{\pi_t}\|_{1,\infty} \leq \frac{1}{(1-\gamma)^2}$. This implies that for all $t$, $\|v_t - \nabla V_{\mu}^{\pi_t}\|_{1, \infty} \leq \frac{4A}{(1-\gamma)^2}$. 

Now consider any $\pi \in \Pi$ and any $\varepsilon > 0$, then  by the covering property  we have that there exists a $\pi' \in \Pi_{\varepsilon}$ such that for all $t$, 
\[|\innerprod{v_t - \nabla V_{\mu}^{\pi_t}}{\pi}| \leq |\innerprod{v_t - \nabla V_{\mu}^{\pi_t}}{\pi'}| + \frac{4A \varepsilon}{(1-\gamma)^2}.\] 
Combining the above with \eqref{eqn:tempeq} completes the proof for all $\pi \in \Pi$. 

To extend the statement to all $\pi$ in the convex hull of $\Pi$, note that purely as a function of $\pi$, $f(\pi)\defeq \max\{\innerprod{v_t - \nabla V^{\pi_t}_\mu}{\pi}, -\innerprod{v_t - \nabla V^{\pi_t}_\mu}{\pi}\}$ is a convex function in $\pi$. Therefore one of its maxima over a convex set must lie at the boundary of the convex set. This implies that establishing the statement of any $\pi \in \Pi$ is sufficient to establish the statement for $\pi \in \mathrm{CH}(\Pi)$.
\end{proof}



\section{Basic results}
\subsection{Duality of $\|\cdot\|_{\infty,1}$ and $\|\cdot\|_{1,\infty}$}
\begin{lemma}\label{lem:dualnorm}
    $\|\cdot\|_{1,\infty}$ and $\|\cdot\|_{\infty,1}$ are dual norms with respect to the matrix dot product, i.e. $$\|Y\|_{1,\infty}=\max_{\|X\|_{\infty,1}=1} \langle X, Y\rangle.$$
\end{lemma}
\begin{proof}[Proof of Lemma~\ref{lem:dualnorm}]
      Let $A_i$ denote the $i^{th}$ row of a matrix $A$. Consider any $X$ be such that $\|X\|_{\infty,1}=1$. Then by Holder's inequality we have
      \begin{align*}
          \langle X, Y\rangle  = \sum_{i=1}^M X_i^\top Y_i \leq \sum_{i=1}^M \|X_i\|_1 \|Y_i\|_\infty \leq \sum_{i=1}^M \|Y_i\|_\infty = \|Y\|_{1,\infty}
      \end{align*}
      Now, construct a matrix $X$ such that $X_{i,j} = \mathbb{I}_{j=\argmax_{j'} |Y_{i,j'}|} \mathrm{sign}(Y_{i,\argmax_{j'} |Y_{i,j'}|}) $ breaking ties arbitrarily for the cases where argmax is non-unique. Clearly, having one unit-sized entry per row, $\|X\|_{\infty,1}=1$. Moreover, observe that the sequence of inequalities stated above is tight for such choice of $X$, because $X_i^\top Y_i = \|Y_i\|_\infty$ and $\|X_i\|_{1}=1$ hold for any $i\in [M]$ by definition of $X$.
  \end{proof}

\subsection{Smoothness of $V^\pi_\mu$ in $\|\cdot\|_{\infty,1}$ (Proof of Lemma~\ref{lem:smooth})}
\begin{proof}[Proof of Lemma~\ref{lem:smooth}]
    Using the performance difference lemma \cite{kakade2002approximately}, we have
    \begin{equation}
    \label{eqn:perfdifflemma}
        V^{\pi'}_\mu - V^\pi_\mu = \frac{1}{1-\gamma} \mathbb{E}_{s\sim d^{\pi'}_\mu } \left[ Q^\pi(\cdot|s)^\top \pi'(s) - Q^\pi(\cdot|s)^\top \pi(s)  \right].
    \end{equation}

    Comparing this to the gradient characterization (Lemma~\ref{lem:grad}), we have
    \begin{align*}
        & | V_\mu^{\pi'} - V_\mu^\pi -  \langle \nabla V^\pi_\mu, \pi'-\pi \rangle |\\
        \leq & \frac{1}{1-\gamma} \left\lvert \mathbb{E}_{s\sim d^{\pi'}_\mu} \left[ Q^\pi(s, \cdot)^\top \pi'(s) - Q^\pi(s, \cdot)^\top \pi(s) \right] - \mathbb{E}_{s\sim d^{\pi}_\mu} \left[ Q^\pi(s, \cdot)^\top \pi'(s) - Q^\pi(s, \cdot)^\top \pi(s) \right] \right\rvert \\
        \leq & \frac{1}{1-\gamma} \|d^{\pi'}_\mu - d^\pi_\mu \|_1 \max_{s\in S} \left\lvert Q^\pi(s, \cdot)^\top \pi'(s) - Q^\pi(s, \cdot)^\top \pi(s)\right\rvert\\
        \leq & \frac{1}{1-\gamma} \|d^{\pi'}_\mu - d^\pi_\mu \|_1 \max_{s\in S} \{\|Q^\pi(s, \cdot)\|_\infty \| \pi'(s) -\pi(s)\|_1\} \\
        \leq & \frac{1}{(1-\gamma)^2} \|d^{\pi'}_\mu - d^\pi_\mu \|_1 \| \pi' -\pi\|_{\infty,1}
    \end{align*}
    where the last inequality follows from $\|Q^\pi(s,\cdot)\|_\infty \leq \frac{1}{1-\gamma}$. Now, define for any policy $\pi$, define $P^\pi(s'|s)  = \sum_{a\in A} P(s'|s,a)\pi(a|s)$ as the associated Markov transition operator. First, for any distribution $d\in \Delta_{S}$, we have
    \begin{align*}
        \| (P^{\pi'} - P^\pi)d \|_1 &=  \sum_{s'\in S} \left\lvert \sum_{a\in A, s\in S} P(s'|s,a) (\pi'(a|s) - \pi(a|s))d(s)  \right\rvert \\
        &\leq \sum_{s'\in S,a\in A, s\in S} P(s'|s,a) |\pi'(a|s) - \pi(a|s) |d(s) \\
        &= \sum_{s\in S, a\in A} |\pi'(a|s) - \pi(a|s) |d(s) \\
        &\leq \max_{s\in S}\sum_{a\in A} |\pi'(a|s) - \pi(a|s) | = \|\pi'-\pi  \|_{\infty,1}
    \end{align*}
    Generalizing this to $t$ successive applications of the Markov operator, we have the following via an inductive argument. Suppose for $t-1$, we have that for all distributions $d \in \Delta_{S}$ we have that $\| ((P^{\pi'})^t - (P^\pi)^t)d \|_1  \leq (t-1), \|\pi'-\pi  \|_{\infty,1}$. Now consider the case for $t$, 
    \begin{align*}
        \| ((P^{\pi'})^t - (P^\pi)^t)d \|_1 &\leq \| ((P^{\pi'})^t - (P^{\pi'})^{t-1}P^{\pi})d \|_1 +   \| ((P^{\pi'})^{t-1}P^{\pi}- (P^\pi)^t)d \|_1 \\
        &= \|(P^{\pi'})^{t-1} (P^{\pi'} - P^{\pi})d \|_1 +   \| ((P^{\pi'})^{t-1}- (P^\pi)^{t-1})P^{\pi}d \|_1 \\
        &\leq  \|(P^{\pi'} - P^{\pi})d \|_1 +   \| ((P^{\pi'})^{t-1}- (P^\pi)^{t-1})P^{\pi}d \|_1 \\
        &\leq  \|\pi'-\pi  \|_{\infty,1} + (t-1)\|\pi'-\pi  \|_{\infty,1} \\
        &= t\|\pi'-\pi  \|_{\infty,1}
    \end{align*}
    Using the definition of $d^\pi_\mu$, we have
    \begin{align}
        \|d^{\pi'}_\mu-d^\pi_\mu\|_1 &\leq (1-\gamma) \sum_{t\geq 0 }\gamma^t \|((P^{\pi'})^t - (P^\pi)^t)\mu\|_1 \nonumber\\
        &\leq (1-\gamma)\sum_{t\geq 0}\gamma^t t \|\pi'-\pi\|_{\infty,1} \nonumber\\
        &= \frac{\gamma}{1-\gamma}\|\pi'-\pi\|_{\infty,1} \label{eqn:temp1}.
    \end{align}
    where the last line uses the identity $\sum_{t\geq 0} \gamma^t t = \frac{\gamma}{(1-\gamma)^2}$, completing the proof. For the second part of the lemma, consider the following. Using Lemma \ref{lem:grad}, we have that,

    \begin{align*}
        \innerprod{\nabla V_{\mu}^{\pi'} - \nabla V_{\mu}^{\pi''}}{\pi} &= \frac{1}{1-\gamma} \left( \mathbb{E}_{s\sim d^{\pi'}_\mu} \left[Q^{\pi'}(s,\cdot)^\top \pi(s) \right] - \mathbb{E}_{s\sim d^{\pi''}_\mu} \left[ Q^{\pi''}(s,\cdot)^\top \pi(s)\right] \right) \\
        & = \frac{1}{1-\gamma} \left(\mathbb{E}_{s\sim d^{\pi'}_\mu} \left[(Q^{\pi'}(s,\cdot) - Q^{\pi''}(s,\cdot))^\top \pi(s) \right] \right) + \\
        & \qquad \frac{1}{1-\gamma} \left(\mathbb{E}_{s\sim d^{\pi'}_\mu} \left[Q^{\pi''}(s,\cdot)^\top \pi(s) \right] - \mathbb{E}_{s\sim d^{\pi''}_\mu} \left[ Q^{\pi''}(s,\cdot)^\top \pi(s)\right] \right).
    \end{align*}
    We now bound the two terms above separately. Applying Cauchy-Schwartz repeatedly and noting the definition of $Q^{\pi}(s,a) = r(s,a) + \gamma \sum_{s' \in \S} P(s'|s,a) V^{\pi}(s)$ we have that 
    \begin{align*}
        \mathbb{E}_{s\sim d^{\pi'}_\mu} \left[(Q^{\pi'}(s,a) - Q^{\pi''}(s,a))^\top \pi(s) \right] &\leq \max_{s,a \in \S \times \A}|(Q^{\pi'}(s,a) - Q^{\pi''}(s,a))| \\
        & = \gamma \max_{s,a \in \S \times \A}\big|\sum_{s' \in  \S} P(s'|s,a) (V^{\pi'}(s) - V^{\pi''}(s))\big| \\
        & \leq \gamma \max_{s\in \S}\big|V^{\pi'}(s) - V^{\pi''}(s)\big| \\
        & = \frac{\gamma}{1 - \gamma} \max_{s\in \S}\biggr|\mathbb{E}_{s'\sim d^{\pi'}_s } \left[ Q^{\pi''}(\cdot|s')^\top \pi'(s') - Q^{\pi''}(\cdot|s')^\top \pi''(s')  \right]\biggr| \\
        & \leq \frac{\gamma}{(1 - \gamma)^2} \|\pi' - \pi''\|_{\infty, 1}.
    \end{align*}
    Here the second last inequality follows from the performance difference lemma \eqref{eqn:perfdifflemma}. Furthermore for the second term we have that, 
    \begin{align*}
        \mathbb{E}_{s\sim d^{\pi'}_\mu} \left[Q^{\pi''}(s,\cdot)^\top \pi(s) \right] - \mathbb{E}_{s\sim d^{\pi''}_\mu} &\left[ Q^{\pi''}(s,\cdot)^\top \pi(s)\right] \leq \frac{1}{(1-\gamma)}\|d^{\pi'}_\mu - d^{\pi''}_\mu\|_{1} \\
        &\leq \frac{\gamma}{(1 - \gamma)^2} \|\pi' - \pi''\|_{\infty, 1}.
    \end{align*}
    where the inequality follows from \eqref{eqn:temp1}. Putting the above statements together we easily see that 
    \[\innerprod{\nabla V_{\mu}^{\pi'} - \nabla V_{\mu}^{\pi''}}{\pi} \leq \frac{2\gamma}{(1 - \gamma)^3} \|\pi' - \pi''\|_{\infty, 1}.\]
\end{proof}

\section{Characterization of the Functional gradient and Hessian of the Value Function (Proofs of Lemmas \ref{lem:grad} and \ref{lem:hess})}
For the sake of this section, we introduce new notation that will help us state analytic derivatives cleanly. For any policy $\pi \in \reals^{S \times A}$ define a function of $\pi$, $T^\pi$ with the signature
$T^\pi: \reals^{S \times A} \rightarrow \reals^{SA\times SA}$ such that for all $((s',a'),(s,a))$, 
$$ T^\pi((s',a'),(s,a)) = P(s'|s,a)\pi(a'|s').$$
Similarly define $P^\pi(s'|s):\reals^{S \times A} \rightarrow \reals^{S\times S}$ such that such that for all $(s',s)$,
$$ P^\pi(s',s) \defeq \sum_{a\in \A} P(s'|s,a) \pi(a|s).$$
Finally define $\mu^\pi:\reals^{S \times A} \rightarrow \reals^{SA}$ such that for all $(s,a)$
$$  \mu^\pi(s,a) = \mu(s)\pi(a|s). $$

For any $(\tilde{s}, \tilde{a})$, we define the following partial derivatives for $(\tilde{s}, \tilde{a})^{\mathrm{th}}$ entry of the input $\pi$ as $\frac{\partial T^\pi}{\partial \pi(\tilde{a}|\tilde{s})}\in \reals^{SA\times SA}$ and $\frac{\partial P^\pi}{\partial \pi(\tilde{a}|\tilde{s})}\in \reals^{S\times S}$ such that 
$$ \frac{\partial T^\pi}{\partial \pi(\tilde{a}|\tilde{s})}((s',a'),(s,a)) = P(\tilde{s}|s,a) \mathbb{I}_{s'=\tilde{s}\wedge a'=\tilde{a}}, \quad \frac{\partial P^\pi}{\partial \pi(\tilde{a}|\tilde{s})}(s', s) = P(s'|\tilde{s},\tilde{a}) \mathbb{I}_{s=\tilde{s}}.  $$ 
It can be seen that for any $\pi \in (\Delta_{A})^S$, $T^{\pi}, P^{\pi}$ are stochastic matrices. Therefore since $\gamma < 1$, we have that $I - \gamma T^{\pi}$ and $I - \gamma P^{\pi}$ are invertible.

Next, consider any start state distribution $\mu\in \Delta_{S}$ and a reward vector $r\in \reals^{S\times A}$. As defined before the steady-state distribution, $d^\pi_\mu: \reals^{S \times A} \rightarrow \reals^{SA}$, the Q-function $Q^\pi: \reals^{S \times A} \rightarrow \reals^{SA}$ and the value function $V^\pi: \reals^{S \times A} \rightarrow \reals^{S}$ are also functions of a policy $\pi$ and thus similar partial derivatives for any $\tilde{s}, \tilde{a}$ can be defined here as well. In this section to make the notation more explicit we define $d^\pi_{\mu,\S} \in \Delta_{S}$ such that for all $s \in S$, $d^\pi_{\mu,\S}(s) = \sum_{a \in \A}(d^\pi_{\mu,\S}(s, a))$. Note that in other parts of the paper we have referred to $d^\pi_{\mu,\S}(s)$ as just $d^\pi_{\mu}(s)$ but since we need to explicitly use the vector $d^\pi_{\mu,\S}(s)$ in this section we make this notation explicit.

Now note by Bellman equations, we have that
\begin{align*}
    d^\pi_\mu = (1-\gamma) \mu^\pi + \gamma T^\pi d^\pi_\mu  \implies d^\pi_\mu = (1-\gamma) (I-\gamma T^\pi)^{-1}\mu^\pi, \\
    d^\pi_{\mu,\S} = (1-\gamma) \mu + \gamma P^\pi d^\pi_{\mu,\S}  \implies d^\pi_{\mu,\S} = (1-\gamma) (I-\gamma P^\pi)^{-1}\mu, \\
    Q^\pi = r + \gamma (T^\pi)^\top Q^\pi \implies Q^\pi = (I-\gamma (T^\pi)^\top)^{-1} r.
\end{align*}
We will now use these notations repeatedly. As a warm-up, we provide a proof the policy gradient lemma (Lemma \ref{lem:grad}) first.

\begin{proof}[Proof of Lemma~\ref{lem:grad}]
First observe that $V^\pi_\mu = (\mu^\pi)^\top Q^\pi$. Now for all $\tilde{s}, \tilde{a}$ we have that 
\begin{align}
    \frac{\partial V^\pi_\mu}{\partial \pi(\tilde{a}|\tilde{s})} & =  \left(\frac{\partial\mu^\pi}{\partial \pi(\tilde{a}|\tilde{s})}\right)^\top Q^\pi  + (\mu^\pi)^\top\frac{\partial Q^\pi}{\partial \pi(\tilde{a}|\tilde{s})} \nonumber\\
    &=\mu(\tilde{s}) Q^\pi(\tilde{s},\tilde{a}) + \gamma (\mu^\pi)^\top \left[(I-\gamma (T^\pi)^\top)^{-1} \frac{\partial (T^\pi)^\top}{\partial \pi(\tilde{a}|\tilde{s})} (I-\gamma (T^\pi)^\top)^{-1} r \right]\nonumber\\
    &=\mu(\tilde{s}) Q^\pi(\tilde{s},\tilde{a}) + \gamma ( (I-\gamma (T^\pi))^{-1}\mu^\pi)^\top \left[\frac{\partial (T^\pi)^\top}{\partial \pi(\tilde{a}|\tilde{s})} (I-\gamma (T^\pi)^\top)^{-1} r \right]\nonumber\\
    &=\mu(\tilde{s}) Q^\pi(\tilde{s},\tilde{a}) + \frac{\gamma}{1-\gamma} (d^\pi_\mu)^\top \frac{\partial (T^\pi)^\top}{\partial \pi(\tilde{a}|\tilde{s})} Q^\pi \nonumber\\
    &=\mu(\tilde{s}) Q^\pi(\tilde{s},\tilde{a}) + \frac{\gamma}{1-\gamma}  Q^\pi(\tilde{s},\tilde{a})\sum_{s\in S, a\in A} P(\tilde{s}|s,a) d^\pi_\mu(s,a) \nonumber\\
    &=\mu(\tilde{s}) Q^\pi(\tilde{s},\tilde{a}) + \frac{\gamma}{1-\gamma}  Q^\pi(\tilde{s},\tilde{a})\sum_{s\in S, a\in A} P(\tilde{s}|s,a) \pi(a|s) d^\pi_\mu(s) \nonumber\\
    &= \frac{d^\pi_{\mu}(\tilde{s})Q^\pi(\tilde{s},\tilde{a})}{1-\gamma} \label{eqn:temp2}
\end{align} 
where we use that $d^\pi_{\mu}(s) = (1-\gamma) \mu(s) + \gamma \sum_{s'\in S, a'\in A} P(s|s',a') \pi(a'|s') d^\pi_\mu(s')$. The statement of the lemma follows immediately now. 
\end{proof}
\begin{proof}[Proof of Lemma~\ref{lem:hess}] Observe for any $s,a,\tilde{s},\tilde{a}$ we have the following statement that follows by product rule,
    \begin{align}
    \frac{\partial (d^\pi_\mu(s) Q^\pi(s,a))}{\partial \pi(\tilde{a}|\tilde{s})} & = d^\pi_\mu(s) \frac{\partial  Q^\pi(s,a)}{\partial \pi(\tilde{a}|\tilde{s})} + \frac{\partial d^\pi_\mu(s) }{\partial \pi(\tilde{a}|\tilde{s})}Q^\pi(s,a). \label{eqn:t1}
    \end{align}
    Further given any $s, a$ we can define $d_{s,a}^{\pi} \in \reals^{SA}$ to be the steady state distribution starting from state $s$ and executing action $a$. In particular the following holds
    \[d^\pi_{s,a} = (1-\gamma) e_{s,a} + \gamma T^\pi d^\pi_\mu  \implies d^\pi_{s,a} = (1-\gamma) (I-\gamma T^\pi)^{-1}e_{s,a},\]
    where $e_{s,a} \in \reals^{SA}$ is the indicator vector of the $(s,a)^{th}$ coordinate. We now have the following for any $s,a,\tilde{s},\tilde{a}$,
    \begin{align*}
    \frac{\partial  Q^\pi(s,a)}{\partial \pi(\tilde{a}|\tilde{s})} &= e_{s,a}^\top \frac{\partial  Q^\pi}{\partial \pi(\tilde{a}|\tilde{s})}\\
    &= \gamma e_{s,a}^\top (1-\gamma (T^\pi)^\top)^{-1}) \frac{\partial  (T^\pi)^\top}{\partial \pi(\tilde{a}|\tilde{s})} (1-\gamma (T^\pi)^\top)^{-1}) r\\
    &= \frac{\gamma}{1-\gamma}(d^\pi_{s,a})^\top \frac{\partial  (T^\pi)^\top}{\partial \pi(\tilde{a}|\tilde{s})} Q^\pi \\
    &= \frac{\gamma}{1-\gamma} Q^\pi(\tilde{s},\tilde{a}) \sum_{s'\in S,a'\in A} P(\tilde{s}| s',a') d^\pi_{s,a}(s',a'). 
\end{align*}
The above in particular implies that for any $\pi'$ and any $s,a,$,
\begin{equation}
\label{eqn:t2}
    \left\langle d^\pi_\mu(s) \frac{\partial  Q^\pi(s,a)}{\partial \pi},\pi'\right\rangle = \mathbb{E}_{s\sim d^\pi_{\mu}} \mathbb{E}_{s'\sim P(\cdot | s,a)} \mathbb{E}_{\tilde{s}\sim d^\pi_{s'}} \left[ Q^\pi(\tilde{s},\cdot)^\top \pi'(\tilde{s}) \right] = F^\pi(s,a|\pi').
\end{equation}
Further, consider the following for any $s, \tilde{s}, \tilde{a}$,
\begin{align*}
    \frac{\partial  d_\mu^\pi(s)}{\partial \pi(\tilde{a}|\tilde{s})} &= e_{s}^\top \frac{\partial  d^\pi_{\mu,S}}{\partial \pi(\tilde{a}|\tilde{s})}\\
    &= (1-\gamma) \gamma e_{s}^\top (I-\gamma P^\pi)^{-1} \frac{\partial  P^\pi }{\partial \pi(\tilde{a}|\tilde{s})} (I-\gamma P^\pi)^{-1} \mu\\
      &= \gamma e_{s}^\top (I-\gamma P^\pi)^{-1} \frac{\partial  P^\pi }{\partial \pi(\tilde{a}|\tilde{s})} d^\pi_{\mu,S}.
      \end{align*}
      Let $\mu_{\tilde{s}, \tilde{a}}' \in \Delta_{S}$ such that for all $s'$,
      \[ \mu_{\tilde{s}, \tilde{a}}'(s') = P(s'|\tilde{s},\tilde{a}) d^\pi_{\mu}(\tilde{s}).\]
      Therefore we have that,
      \begin{align*}
        \frac{\partial  d_\mu^\pi(s)}{\partial \pi(\tilde{a}|\tilde{s})} &= \gamma e_{s}^\top (I-\gamma P^\pi)^{-1} \mu_{\tilde{s}, \tilde{a}}' \\
        &= \frac{\gamma}{1-\gamma} e_{s}^\top d^{\pi}_{\mu_{\tilde{s}, \tilde{a}}'} \\
     &= \frac{\gamma}{1-\gamma} \sum_{s'\in S}d^\pi_{s'}(s) P(s'|\tilde{s},\tilde{a}) d^\pi_{\mu}(\tilde{s}). 
\end{align*}
The above in particular implies that for any $\pi''$ and any $\tilde{s}, \tilde{a}$,
\begin{equation}
\label{eqn:t3}
    \left\langle Q^\pi \frac{\partial  d_\mu^\pi}{\partial \pi(\tilde{a}|\tilde{s})},\pi''\right\rangle = \mathbb{E}_{\tilde{s}\sim d^\pi_{\mu}} \mathbb{E}_{s'\sim P(\cdot | \tilde{s},\tilde{a})} \mathbb{E}_{s\sim d^\pi_{s'}} \left[ Q^\pi(s,\cdot)^\top \pi''(s) \right] = F^\pi(\tilde{s},\tilde{a}|\pi'').
\end{equation}
Combining \eqref{eqn:t1}, \eqref{eqn:t2}, \eqref{eqn:t3} completes the proof.
\end{proof}
\section{Sampling subroutines}
\subsection{Properties of Q-sampler (Algorithm~\ref{alg:qsamp})}
\begin{proof}[Proof of Lemma~\ref{lem:qsamp}]
    Let $E_t$ be the event that the first loop terminates at the $t^{th}$ iteration. Then
    $$ \mathbb{P}(s_t=s) = \sum_{\tau=0}^\infty \mathbb{P}(E_\tau)\mathbb{E}\left[\mathbf{1_{s_\tau=s}} |\pi\right] = \sum_{\tau=0}^\infty (1-\gamma)\gamma^\tau \mathbb{E}\left[\mathbf{1_{s_\tau=s}} |\pi,E_\tau \right] = d_\pi(s).$$
    Let $F_t$ be the event that the second loop terminates at the $t^{th}$ iteration. Then
    \begin{align*}
        \mathbb{E}\left[R|s_t,a_t\right] & =\sum_{\tau=0}^\infty \mathbb{P}(F_\tau)\mathbb{E}\left[ \frac{r(s_{t+\tau},a_{t+\tau})}{1-\gamma} \bigg|\pi,s_t,a_t, F_\tau \right]\\
        & = \sum_{\tau=0}^\infty (1-\gamma)\gamma^\tau \mathbb{E}\left[\frac{r(s_{t+\tau},a_{t+\tau})}{1-\gamma} \bigg|\pi,s_t,a_t, F_\tau \right] \\
        & =  \mathbb{E}\left[\sum_{\tau=0}^\infty \gamma^\tau r(s_{t+\tau},a_{t+\tau}) \bigg|\pi,s_t,a_t, F_\tau \right] \\
        &=  Q^\pi(s_t,a_t).
    \end{align*}
    Now, taking the marginal over the choice of $a_t$, we have for any $a\in A$
    $$\mathbb{E}\left[\widehat{Q}(a)\Big|s_t,\pi\right] = A\mathbb{E}\left[R|s_t,a_t=a\right] \mathbb{P}(a_t=a) = Q^\pi(s_t,a).$$
    Since $E_t$ and $F_t$ are geometric random variables with $(1-\gamma)$ probability of termination, the expected survival length of each is $\frac{1}{1-\gamma}$.
\end{proof}

\subsection{Properties of H-sampler (Algorithm~\ref{alg:qsamp})}
\begin{proof}[Proof of Lemma~\ref{lem:hsamp_sub}]
Since the sampling procedure of $s$ coincides with the first phase of Algorithm~\ref{alg:qsamp}, analogously we have that $s\sim d^\pi_\mu$. Similarly, by comparisons to the second phase of Algorithm~\ref{alg:qsamp}, we have $\mathbb{E}[R|s_{t'},a_{t'}] = Q^\pi(s_{t'},a_{t'})$. Taking the marginal over the choice of $a_{t'}$, we have for any $\pi'$ that 
    $$ \mathbb{E}\left[\widehat{H}^\top \pi'(s')| s,\pi\right] = F^\pi(s,\cdot|\pi') .$$
\end{proof}
\begin{proof}[Proof of Lemma~\ref{lem:hsamp}]
  The first part of the claim is a synthesis of Lemma~\ref{lem:hsamp_sub} and the characterization of the policy Hessian in Lemma~\ref{lem:hess}. The episode length is bounded as each of the three rejection sampling phases has a survival length of $\frac{1}{1-\gamma}$.
\end{proof}

\section{Local-to-Global Lemmas (Proofs of Lemmas \ref{lem:transfernu} and \ref{lem:transferpi})}
\subsection{For an exploratory distribution}
\begin{proof}[Proof of Lemma~\ref{lem:transfernu}]
    Consider any $\pi\in \mathrm{CH(\Pi)}$. Due to performance difference lemma \citep{ajks}, we have 
    \begin{align*}
        V^* - V^\pi =& \frac{1}{1-\gamma} \mathbb{E}_{s\sim d^{\pi^*}} \left[Q^\pi(s,\cdot)^\top \pi^*(s) - Q^\pi(s,\cdot)^\top \pi(s) \right] \\
        \leq& \frac{1}{1-\gamma} \mathbb{E}_{s\sim d^{\pi^*}} \left[\max_{a\in A}Q^\pi(s,a) - Q^\pi(s,\cdot)^\top \pi(s) \right] \\
        \leq& \frac{1}{1-\gamma}\left\| \frac{d^{\pi^*}}{d^\pi_\mu} \right\|_\infty \mathbb{E}_{s\sim d^{\pi}_\mu} \left[\max_{a\in A}Q^\pi(s,a) - Q^\pi(s,\cdot)^\top \pi(s) \right] \\
        \leq& \frac{1}{(1-\gamma)^2}\left\| \frac{d^{\pi^*}}{\mu} \right\|_\infty \bigg( \min_{\pi'\in\Pi}\mathbb{E}_{s\sim d^{\pi}_\mu} \left[\max_{a\in A}Q^\pi(s,a) - Q^\pi(s,\cdot)^\top \pi'(s) \right] \\
        &+\max_{\pi'\in\Pi} \mathbb{E}_{s\sim d^{\pi}_\mu} \left[Q^\pi(s,\cdot)^\top (\pi'(s) - \pi(s)) \right]\bigg) \\
        \leq& \frac{1}{(1-\gamma)^2}\left\| \frac{d^{\pi^*}}{\mu} \right\|_\infty \bigg( \epsilon_{\Pi,\mu} + (1-\gamma) \max_{\pi'\in \Pi} \langle \nabla V^\pi_{\mu}, \pi'-\pi \rangle\bigg),
    \end{align*}
    where the last line follows from Lemma~\ref{lem:grad} and the definition of $\epsilon_{\Pi,\mu}$. 
\end{proof}

\subsection{For an exploratory policy class}
\begin{proof}[Proof of Lemma~\ref{lem:transferpi}]
    Consider any $\pi\in \mathrm{CH(\Pi)}$. Due to performance difference lemma \eqref{eqn:perfdifflemma}, we have 
    \begin{align*}
        V^* - V^\pi =& \frac{1}{1-\gamma} \mathbb{E}_{s\sim d^{\pi^*}} \left[Q^\pi(s,\cdot)^\top \pi^*(s) - Q^\pi(s,\cdot)^\top \pi(s) \right] \\
        \leq& \frac{1}{1-\gamma} \mathbb{E}_{s\sim d^{\pi^*}} \left[\max_{a\in \A}Q^\pi(s,a) - Q^\pi(s,\cdot)^\top \pi(s) \right] \\
        \leq& \frac{1}{1-\gamma}\left\| \frac{d^{\pi^*}}{d^\pi} \right\|_\infty \mathbb{E}_{s\sim d^{\pi}} \left[\max_{a\in A}Q^\pi(s,a) - Q^\pi(s,\cdot)^\top \pi(s) \right] \\
        \leq& \frac{1}{1-\gamma}\left\| \frac{d^{\pi^*}}{d^\pi} \right\|_\infty \bigg( \min_{\pi'\in\Pi}\mathbb{E}_{s\sim d^{\pi}} \left[\max_{a\in A}Q^\pi(s,a) - Q^\pi(s,\cdot)^\top \pi'(s) \right] \\
        &+\max_{\pi'\in\Pi} \mathbb{E}_{s\sim d^{\pi}} \left[Q^\pi(s,\cdot)^\top  (\pi'(s) - \pi(s)) \right]\bigg) \\
        \leq& \frac{1}{1-\gamma}\left\| \frac{d^{\pi^*}}{d^\pi} \right\|_\infty \bigg( \epsilon_{\Pi} + (1-\gamma) \max_{\pi'\in \Pi} \langle \nabla V^\pi, \pi'-\pi \rangle\bigg),
    \end{align*}
    where the last line follows from Lemma~\ref{lem:grad} and the definition of $\epsilon_{\Pi}$.
\end{proof}

%
%

\end{document}